%% file: paper.tex
\documentclass{article}

\PassOptionsToPackage{compress}{natbib}

\usepackage{natbib}
\usepackage[main, final]{neurips_2026}

\usepackage[utf8]{inputenc} 
\usepackage[T1]{fontenc}    
\usepackage{hyperref}       
\usepackage{url}            
\usepackage{booktabs}       
\usepackage{amsfonts}       
\usepackage{nicefrac}       
\usepackage{microtype}      
\usepackage{xcolor}         

\usepackage{amsthm}
\usepackage{subcaption}

\usepackage{amsmath}
\usepackage{hyperref}
\usepackage{url}
\usepackage{newfloat}
\usepackage{listings}
\usepackage{amsmath}
\usepackage{pifont}
\usepackage[noend]{algpseudocode}
\usepackage{algorithm}
\usepackage{graphicx}
\usepackage{adjustbox}
\usepackage{enumitem}
\usepackage[table]{xcolor}
\usepackage{colortbl}
\usepackage{array}
\usepackage{wrapfig}
\usepackage{xurl}
\usepackage{booktabs} 
\hypersetup{
    colorlinks=true,
    linkcolor=red,
    filecolor=magenta,      
    urlcolor=magenta,
    citecolor=teal,
    pdftitle={MoEZip: Routing-Aware KV Cache Compression for Sparse Mixture-of-Experts LLMs},
    pdfpagemode=FullScreen,
    }
\usepackage{listings}
\usepackage{subcaption}
\usepackage{verbatim}
\usepackage{algpseudocode}
\usepackage{bm}
\usepackage{amssymb}
\usepackage{multirow}

\newtheorem{proposition}{Proposition}

\definecolor{Ivory}{HTML}{FFFFF0}
\definecolor{Lavender}{HTML}{F5F5FA}
\definecolor{AliceBlue}{HTML}{F0F8FF}
\definecolor{Seashell}{HTML}{FFF5EE}

\title{MoEZip: Routing-Aware KV Cache Compression for Sparse Mixture-of-Experts LLMs}

\author{%
  Minsang Kim \\
  Korea University\\
  SK Telecom\\
  \texttt{kmswin1@korea.ac.kr} \\
  \And
  Seung Jun Baek\thanks{Corresponding Author} \\
  Korea University\\
  \texttt{sjbaek@korea.ac.kr} \\
}

\begin{document}
\maketitle
\begin{abstract}
Mixture-of-Experts (MoE) language models scale model capacity efficiently by activating a small subset of parameters per token. However, in long-context inference, memory remains a major bottleneck because the KV cache grows linearly with sequence length. Existing KV cache compression methods reduce this cost for dense transformers, but overlook a problem specific to MoE: KV eviction can change which experts are selected during answer generation, degrading the generation quality. We propose MoEZip, a routing-aware KV cache compression for sparse MoE language models. To estimate the sensitivity of expert routing to perturbation, we propose a novel metric based on the Fisher information matrix of routing distributions. MoEZip combines this metric with attention weights for routing-aware scoring of contexts. Finally, we compose a complementary scoring system that combines eviction methods for dense and sparse-MoE architectures. Experiments show that MoEZip achieves robust performance under aggressive cache compression on various benchmarks. At a KV retention ratio of 10\%, MoEZip improves over the SoTA baseline by +24 and +20 points on average on LongBench and RULER, corresponding to 2.3$\times$ and 4.9$\times$ higher scores, respectively. The source code is available at \href{https://github.com/kmswin1/MoEZip}{https://github.com/kmswin1/MoEZip}.
\end{abstract}

\input{introduction}

\input{preliminaries}

\input{method}

\input{experiment}

\input{related}

\input{conclusion}

\bibliography{paper}

\input{appendix}

\input{checklist}

\end{document}

%% file: introduction.tex
\section{Introduction}\label{sec:intro}

Large language models (LLMs)~\cite{brown2020language,grattafiori2024llama} have enabled substantial advances across a wide range of real-world applications, including chatbot assistants~\cite{achiam2023gpt,bai2022training,comanici2025gemini}, coding agents~\cite{li2025minimax,zeng2026glm}, robot agents~\cite{intelligence2025pi_,kim2024openvla}, and broader LLM-based agent systems~\cite{shinn2023reflexion,madaan2023self,yao2022react}. Mixture-of-Experts (MoE) architectures have emerged as an effective solution for scaling LLM capacity while controlling inference cost~\cite{fedus2022switch,du2022glam,liu2024deepseek}. By activating only a small subset of parameters for each token, MoE models can scale total capacity without incurring the full computation cost of dense models. This favorable scaling property has made MoE a promising architecture for large-scale deployment~\cite{deepseekai2025deepseekv32,team2025kimi}.

However, efficient inference remains challenging in long-context and agentic settings. Many real-world applications require processing long interaction histories~\cite{jimenez2023swe,zhou2023webarena}, retrieved documents~\cite{lewis2020retrieval}, tool outputs~\cite{schick2023toolformer}, or multimodal observations~\cite{liu2023libero,zheng2024gpt}. In these scenarios, the key bottleneck is often memory rather than activated parameter count: the KV cache grows linearly with sequence length and quickly becomes a major obstacle for scalable inference.

To address this issue, a body of work recently studied KV cache compression and eviction methods, which reduce memory usage while preserving generation quality~\cite{li2024snapkv,cai2024pyramidkv,zhang2023h2o,kim2025kvzip,zandieh2025turboquant}. Existing approaches span diverse strategies, from attention-based token selection methods such as SnapKV~\cite{li2024snapkv} to methods based on context reconstruction such as KVZip~\cite{kim2025kvzip}. Although these methods have shown promising results on dense transformers, applying them to MoE models can be problematic. In sparse-MoE architectures, changes in KV representation due to eviction can affect router decisions that determine which experts are activated. This has been shown to aggravate performance degradation under high compression. KV compression can therefore incur MoE-specific problems, which are not captured by eviction methods designed for dense architectures.

\noindent\textbf{Contributions.}
In this paper, we consider KV cache compression for sparse MoE models from the perspective of \emph{routing fragility}: it is defined as the sensitivity of routers due to perturbation induced by KV eviction. To estimate fragility, we propose a novel metric based on the Fisher information matrix of routing distributions. We derive the metric for two popular types of MoE routing: softmax-type and sigmoid-type routing. The metric is combined with the attention weights of instruction tokens to produce the fragility score that estimates the importance of context positions. Finally, we propose a scoring system that combines the importance scores from compression methods for dense and sparse-MoE architectures. Our complementary approach enables compression robust to perturbation on both context saliency and routing stability. Experiments show that MoEZip achieves robust performance on various benchmarks such as LongBench, NIAH, and RULER. MoEZip is especially strong at aggressive compression, e.g., at a KV retention ratio of 10\%, MoEZip outperforms the current SoTA method, KVZip \cite{kim2025kvzip}, by +24 and +20 points on average on LongBench and RULER, corresponding to 2.3$\times$ and 4.9$\times$ higher scores, respectively. Further analysis shows that MoEZip achieves good performance for both softmax-type and sigmoid-type routing models and is compatible with diverse compression methods originally developed for dense Transformer LLMs.

%% file: preliminaries.tex
\section{Preliminaries: Mixture-of-Experts Models and Prefix KV Compression}
\label{sec:prelim}

\noindent\textbf{Sparse MoE language models with prefix KV cache.}
Let $x_{1:L}$ denote a prefix sequence and consider autoregressive decoding at step $t$.
We focus on sparse Mixture-of-Experts (MoE) transformer language models, where a subset of transformer layers,
denoted by $\mathcal{M}$, replaces the standard feed-forward block with an MoE block.
At each MoE layer $\ell\in\mathcal{M}$, a router selects a small subset of experts from a pool of $E$ experts.

Meanwhile, autoregressive decoding uses a prefix KV cache to avoid recomputing attention over past tokens.
At transformer layer $\ell$ and attention head $h$, the cached prefix key/value pairs are
\[
\{(k_{\ell,h,i}, v_{\ell,h,i})\}_{i=1}^{L}.
\]
Given the current query vector $q_{t,\ell,h}$, the causal attention probability to prefix position $i$ is
\begin{equation}
\alpha_{t,\ell,h,i}
=
\mathrm{softmax}_i\!\left(
\frac{\langle q_{t,\ell,h}, k_{\ell,h,i}\rangle}{\sqrt{d}}
\right),
~~
\sum_{i=1}^{L}\alpha_{t,\ell,h,i}=1,
\label{eq:prelim_alpha_head}
\end{equation}
where causal masking is implied.
Each head produces the pooled value
\begin{equation}
u_{t,\ell,h}
=
\sum_{i=1}^{L}\alpha_{t,\ell,h,i}\,v_{\ell,h,i}.
\label{eq:prelim_attn_pool_head}
\end{equation}
After merging multi-head outputs and applying the output projection, we denote the resulting post-attention hidden state by $u_{t,\ell}$. This representation is then fed to the downstream blocks, including the MoE router, when layer $\ell$ is sparse.

\noindent\textbf{Prefix KV Eviction.} Although the prefix KV cache accelerates autoregressive decoding, its memory cost grows linearly with the prefix length. To reduce memory and decoding overhead, recent methods compress the cache by evicting or pruning a subset of prefix KV entries. Formally, let $\mathrm{KV}$ denote the full prefix cache and $\mathrm{KV}^{(S)}$ denote the retained cache indexed by $S\subseteq\{1,\dots,L\}$. Let $f_\text{LM}(\cdot|C)$ denote the LLM generation conditional on cache $C$. The goal of prefix KV eviction is to find a subset $S$ such that decoding with the pruned cache remains close to decoding with the full cache:
\begin{equation}
f_{\mathrm{LM}}(\,\cdot \mid \mathrm{KV}^{(S)}\,)
\approx
f_{\mathrm{LM}}(\,\cdot \mid \mathrm{KV}\,).
\label{eq:prelim_kv_goal}
\end{equation}
In dense transformers, this problem is often addressed by retaining KV entries that receive large attention mass. However, in sparse MoE models, cache eviction has an additional mode of failure: it can perturb router decisions and thereby change which experts are selected during decoding.

\noindent\textbf{Routing functions.}  In a sparse MoE layer $\ell\in\mathcal{M}$, the expert routing for each token is determined as follows. The routing is a function of logits denoted by
\begin{equation}
z_{t,\ell}\in\mathbb{R}^{E}.
\end{equation}
where we assume a linear router head $
z_{t,\ell} = R_\ell u_{t,\ell}$
where $u_{t,\ell}$ denotes the post-attention hidden state defined in \eqref{eq:prelim_attn_pool_head}. We consider two popular types of expert routing: softmax and sigmoid. 

\noindent\emph{(1) Softmax-type routing.} We define the routing vector $g$ and the associated distribution $p$:
\begin{equation}
g_{t,\ell}
:=
\mathrm{softmax}(z_{t,\ell})
\in \Delta^{E-1},\quad 
p(x|z_{t,\ell}) := x^\top g_{t,\ell}=\frac{\exp(x^\top z_{t,\ell})}{\sum_e \exp(z_{t,\ell,e})}
\label{eq:prelim_softmax_gate}
\end{equation}
where $g$ induces the distribution $p$ over $E$-dimensional one-hot vector $x$. Thus, under distribution $p$, expert $e$ is selected with probability $g_{t,\ell,e}$ with the $e$-th component of $x$ being 1 under this categorical family.

\noindent\emph{(2) Sigmoid-type routing.} A sigmoid gate is applied independently to each expert logit:
\begin{equation}
g_{t,\ell}
:=
\sigma(z_{t,\ell})
\in (0,1)^E,\quad p(x|z_{t,\ell}) := \prod_{e=1}^E g^{x_e}_{t,\ell,e} (1-g_{t,\ell,e})^{1-x_e}
\label{eq:prelim_sigmoid_gate}
\end{equation}
where $\sigma(\cdot)$ is applied elementwise, and $x$ is a 0-1 vector in $\{0,1\}^E$. Under distribution $p$, router $e$ is selected with probability $g_{t,\ell,e}$ independently of others under this Bernoulli family. 

\noindent\textbf{Router perturbation.}\label{sec:router_perturb} The router logits $z_{t,\ell}$ depend on the attention output $u_{t,\ell}$. Pruning prefix KV entries therefore perturbs $u_{t,\ell}$, which propagates to $z_{t,\ell}$, the routing vector $g_{t,\ell}$, and the routing distribution $p$. The resulting shift in expert selection can degrade generation quality. We design an experiment to observe this effect.

\begin{wrapfigure}[13]{r}{0.42\textwidth}
    \vspace{-1.0cm}
    \centering
    \includegraphics[width=0.42\textwidth]{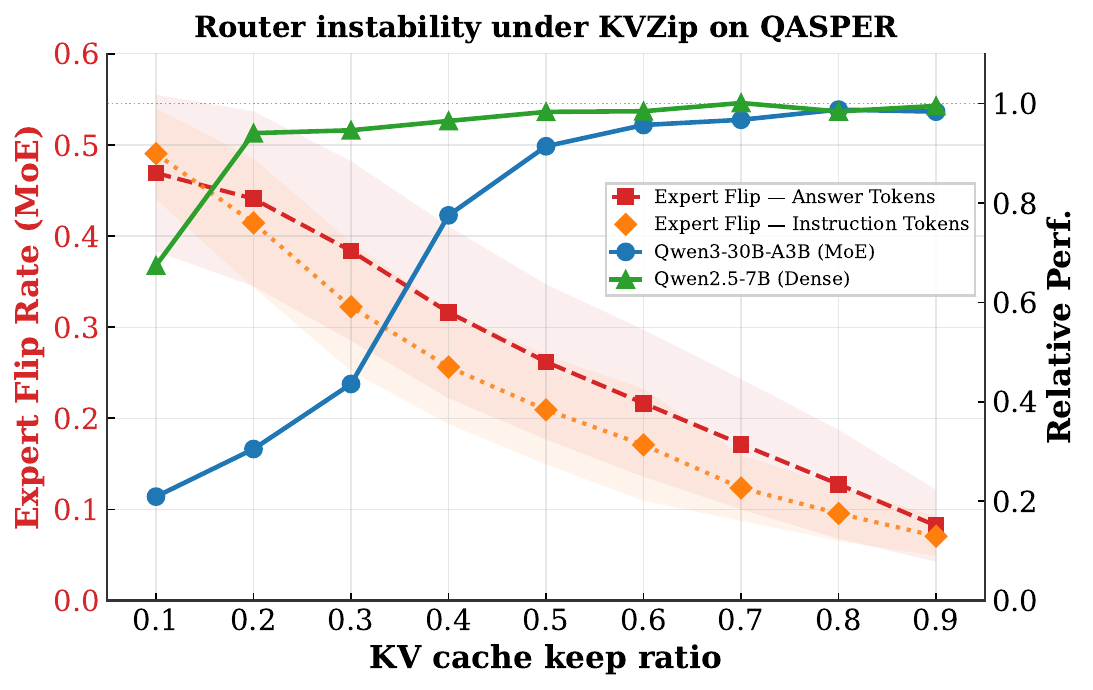}
    \vspace{-0.6cm}
    \caption{\textbf{Router instability is associated with answer-quality degradation.} As KV compression becomes more aggressive, expert selection changes more often, and performance drops accordingly.}
    \label{fig:router_flip}
    \vspace{-0.9cm}
\end{wrapfigure}

Let \(e_{t,\ell}^{\mathrm{full}}\) and \(e_{t,\ell}^{\mathrm{comp}}\) denote the selected expert for token \(t\) at MoE layer \(\ell\) under the full and compressed KV caches, respectively.  Thus, \(e_{t,\ell}^{\mathrm{full}}\) and \(e_{t,\ell}^{\mathrm{comp}}\) differ \emph{only} because of the underlying context KV cache. Let \(\mathcal{I}\) and \(\mathcal{A}\) denote the instruction- and answer-token positions. For any token subset \(\mathcal{T} \subseteq \mathcal{I}\cup\mathcal{A}\), we define the expert-flip rate as
\begin{equation}
\mathrm{Flip}(\mathcal{T})
=
\frac{1}{|\mathcal{T}|\,|\mathcal{M}|}
\sum_{t \in \mathcal{T}}
\sum_{\ell \in \mathcal{M}}
\mathbf{1}
\left[
e_{t,\ell}^{\mathrm{comp}}
\neq
e_{t,\ell}^{\mathrm{full}}
\right],
\label{eq:flip_rates}
\end{equation}
where \(\mathcal{M}\) denotes the set of MoE layers, and we write \(\mathrm{Flip}(\mathcal{I})\) and \(\mathrm{Flip}(\mathcal{A})\) for the rates restricted to instruction- and answer-token positions, respectively. 

To isolate the effect of KV eviction on routing from generation drift, we feed the \emph{same} token sequence to the model under two conditions, with and without prefix KV eviction, where the input concatenates a context prefix with the instruction and teacher-forced ground-truth answer tokens, and compare the selected experts at matched positions. Figure~\ref{fig:router_flip} reports $\mathrm{Flip}(\mathcal{I})$, $\mathrm{Flip}(\mathcal{A})$, and relative performance, which is the task score under compressed KV divided by the full-cache score, on QASPER across keep ratios. We make the following observations:
\begin{enumerate}
    \item The quality of the MoE model drops sharply as compression becomes aggressive, while the dense Qwen2.5-7B-Instruct~\cite{qwen25} degrades far less under the same KVZip compression~\cite{kim2025kvzip}.
    \item The answer-token flip rate \(\mathrm{Flip}(\mathcal{A})\) rises with the quality drop. This indicates a highly suboptimal expert routing and is likely the key cause of output degradation.
    \item  The instruction-token flip rate \(\mathrm{Flip}(\mathcal{I})\) tracks \(\mathrm{Flip}(\mathcal{A})\) closely, and  two rates converge further under aggressive compression.
\end{enumerate}

The last observation is crucial: \(\mathrm{Flip}(\mathcal{A})\) can only be measured after answer generation under the evicted cache, whereas \(\mathrm{Flip}(\mathcal{I})\) is available immediately after processing the instruction, before KV eviction and answer generation. Thus, instruction-token signals can serve as a useful predictor of routing failure. Later, we leverage this signal for MoE-specific KV compression.

%% file: method.tex
\section{Method: Routing-Aware Prefix KV Compression}\label{sec:method}

\noindent\textbf{Overview.} The expert selection of future answer tokens is unavailable at KV eviction time. Thus, MoEZip uses instruction-token router states as a proxy for estimating routing fragility. Specifically, we measure the routing fragility of instruction tokens via the Fisher trace of the router distribution, and combine this fragility with instruction-to-context attention to score the routing-aware importance of each context position. Context positions are ranked by this score, and the top-scoring positions are reserved as \emph{routing-aware anchors}. Within the per-(layer, head) KV retention budget $B$, MoEZip first allocates a portion to anchors and then lets a \emph{base scorer} (e.g., KVZip~\cite{kim2025kvzip}, SnapKV~\cite{li2024snapkv}) select the remaining positions by contextual saliency. Thus, MoEZip changes only the composition of the retained cache while keeping the total cache size and decoding procedure unchanged.

\begin{figure*}[t!]
    \centering
    \includegraphics[width=1.0\linewidth]{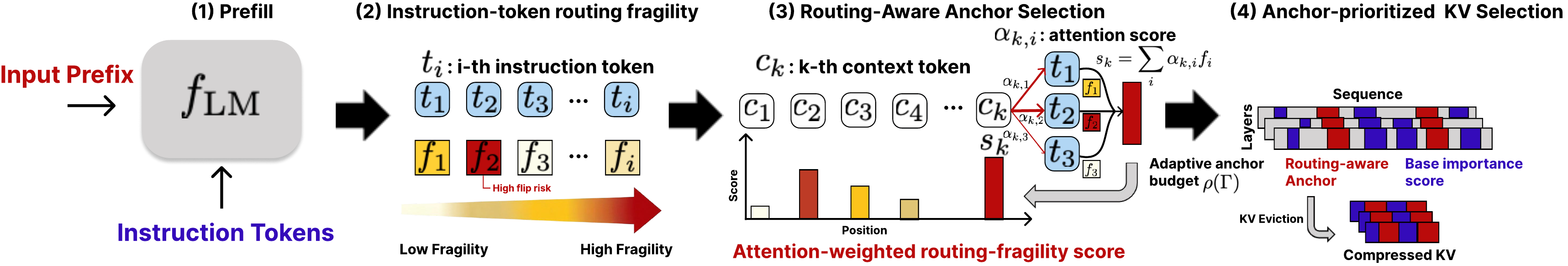}
    \caption{\textbf{Overview of MoEZip.} MoEZip estimates routing fragility from instruction-token router states, transfers it to context positions via instruction-to-context attention, and reserves the top-scoring positions as routing-aware anchors. The base scorer fills the remaining budget, and thus the final cache size and decoding procedure remain unchanged.}
\label{fig:method_overview}
    \label{fig:moezip}
\end{figure*}

\noindent\textbf{Router deviation under logit perturbations.}
To quantify the sensitivity of routing to eviction-induced perturbations, we consider the KL divergence
\begin{equation}
D\!\left(p(\cdot|z)\,\|\,p(\cdot|z+\delta z)\right)
\label{eq:method_divergence}
\end{equation}
for perturbation $\delta z$ on router logit $z$. The divergence captures how much the router distribution changes under a perturbation to its input logits. However, the perturbation $\delta z$ is unavailable before KV selection, since it depends on which KV entries are removed. $\delta z$ depends on which context positions are removed and how those removals cause representation shift propagating into the input router logits. Thus, \eqref{eq:method_divergence} is not directly usable as a pre-selection criterion. Instead, we use a tractable direction-agnostic proxy based on the current router distribution. The following proposition introduces such a metric, which is also easy to compute.

\begin{proposition}
\label{prop:fisher_trace}
Suppose $\delta z\in\mathbb{R}^E$ is a random vector with \(
\mu_{\delta z} = \mathbf{0}\) and \(\Sigma_{\delta z}= \sigma^2 I\). 
For sufficiently small $\delta z$, we have that
\begin{equation}
\mathbb E_{\delta z}\!\left[
D\!\left(p(\cdot|z)\,\|\,p(\cdot|z+\delta z)\right)
\right]
\propto
\operatorname{trace}\!\left(F(z)\right)
\label{eq:expected_kl_trace}
\end{equation}

where $F(z)$ is the Fisher information matrix of $p(\cdot|z)$. For router logit $z_{t,\ell}$ of token $t$ at layer $\ell$, we define Fisher-trace metric
\begin{equation}
f_{t,\ell}
:=\operatorname{trace}\!\left(F(z_{t,\ell})\right).\label{eq:fragility_def}
\end{equation}
For \textbf{softmax-type routing}, we have that, using $g$ in \eqref{eq:prelim_softmax_gate},
\begin{equation}
f_{t,\ell} = 1-\sum_{e=1}^{E} g_{t,\ell,e}^{2},
\label{eq:fragility_softmax_closed}
\end{equation}
For \textbf{sigmoid-type routing}, we have that, using $g$ in \eqref{eq:prelim_sigmoid_gate},
\begin{equation}
f_{t,\ell} = \sum_{e=1}^{E} g_{t,\ell,e}(1-g_{t,\ell,e}).
\label{eq:fragility_sigmoid_closed}
\end{equation}
\end{proposition}
The proof is in Appendix~\ref{appendix:router_fisher}. This perturbation model is used to motivate a direction-agnostic local sensitivity proxy, not to characterize the exact distribution of eviction-induced router-logit shifts. Below, we explain the rationale behind the Fisher-trace metric \eqref{eq:fragility_def}.

Firstly, KL divergence admits a quadratic expansion based on the Fisher information matrix $F(z)$~\cite{amari2000methods}:
\begin{equation}
D\!\left(p(\cdot|z)\,\|\,p(\cdot|z+\delta z)\right)\approx \frac{1}{2}(\delta z)^\top F(z)\,(\delta z).
\label{eq:kl_local_expand}
\end{equation} up to the order of $\|\delta z\|^2$. $F(z)$ is positive semi-definite, and thus it has $E$ orthogonal eigenvectors with non-negative eigenvalues. Averaging the quadratic form in \eqref{eq:kl_local_expand} over perturbation directions yields a quantity proportional to the mean eigenvalue of $F$, and hence to $\operatorname{trace}(F)$. Thus, $f$ can be viewed as the sensitivity averaged over all possible perturbation directions. Secondly, the following holds in general for Fisher information matrix:
\begin{align}
\text{trace}(F)\!=\mathbb{E}\left[\sum_e \left(\frac{\partial \log~p(x|z)}{\partial z_e}\right)^2\right] \!= \mathbb{E}\|s\|^2,~~ s:=\left(\frac{\partial \log~p(x|z)}{\partial z_1},\dots,\frac{\partial \log~p(x|z)}{\partial z_E}\right)
\end{align}
i.e., Fisher-trace is the expected magnitude of \emph{score function} $s$ \cite{hyvarinen2005estimation}. The score $s$ is a vector of sensitivity of (log-)distribution to latent parameters. Thus, Fisher-trace is the magnitude of sensitivity of the routing distribution to change in logit $z$. Thirdly, it is easy to compute the Fisher-trace. The complexity of computing $F(z)$ or divergence \eqref{eq:kl_local_expand} is $O(E^2)$. The expressions in \eqref{eq:fragility_softmax_closed} and \eqref{eq:fragility_sigmoid_closed} incur only $O(E)$. Detailed discussion of the overall computational complexity is provided in Appendix \ref{appendix:complexity}.

Alternatively, one can consider the \emph{worst-case perturbation}, whereas the Fisher trace concerns the ``average'' case. If $\delta z$ aligns with the eigenvector with the largest eigenvalue or $\lambda_{\max}(F)$, the divergence \eqref{eq:kl_local_expand} is maximized. We discuss this in Appendix \ref{appendix:worst}, where we derive some bounds on $\lambda_{\max}$ and compare the performances of the average $(\text{trace})$ and the worst-case  $(\lambda_{\max})$ metrics. 

In summary, the Fisher trace $f_{t,\ell}$ is a tractable metric that quantifies how brittle the router's expert selection is at token $t$ and layer $\ell$.

\noindent\textbf{Routing-fragility score.}
\label{sec:q2c_score} Measuring routing fragility on answer tokens would be ideal; however, those tokens are not yet available when KV eviction is performed. Instead, we propose measuring the routing fragility at the instruction tokens. This is a reasonable proxy; as observed in Fig. \ref{fig:router_flip}, answer tokens are affected to a similar degree to instruction tokens because both tokens pertain to a common instruction. A context position is considered important for routing stability when it is strongly attended by instruction tokens with high $f_{t,\ell}$; we rank context positions by this routing-aware score and retain the top-scoring positions. We first aggregate attention over heads:
\begin{equation}
\bar{\alpha}_{t,\ell,i}
=
\max_h \alpha_{t,\ell,h,i}.
\label{eq:head_reduce}
\end{equation}
This max-aggregation retains strong instruction-to-context attention whenever at least one head places substantial mass on position $i$. We then define an attention-propagated routing-fragility score for each context position:
\begin{equation}
r(i) := \sum_{\ell\in\mathcal{M}}\sum_{t\in\mathcal{I}} \bar{\alpha}_{t,\ell,i} f_{t,\ell}.
\label{eq:r_score}
\end{equation}
Intuitively, Eq.~\eqref{eq:r_score} highlights context positions that are both attended to by instruction tokens and feed routers with fragile decisions. 
The KV pairs with the highest values of $r(i)$ are deemed important for routing stability and will be called \emph{anchors}.

\noindent\textbf{Context Scoring with Anchors.}
\label{sec:adaptive_budget}
Let $B$ denote per-layer/head retention budget. MoEZip selects $B_f\leq B$ context positions as anchors based on scores $r(i)$, whereas the rest of $B-B_f$ positions are scored by the base compressor. $B_f$ is adaptively determined by sample-level fragility. Specifically, we threshold the token-level Fisher score

\begin{equation}
\mathbf{1}_{t,\ell} \;:=\; \mathbf{1}\!\left[f_{t,\ell} > \tau\right].
\label{eq:fragile_indicator}
\end{equation}
with a fixed $\tau>0$, and aggregate it into a sample-average routing fragility
\begin{equation}
\Gamma =
\frac{1}{|\mathcal{M}|\,|\mathcal{I}|}
\sum_{\ell\in\mathcal{M}}\sum_{t\in\mathcal{I}} \mathbf{1}_{t,\ell}.
\label{eq:gamma_def}
\end{equation}
which we map to an adaptive anchor budget
\begin{equation}
\rho(\Gamma)
=
\rho_{\min}
+
(\rho_{\max}-\rho_{\min})\Gamma,
\qquad
B_f
=
\left\lfloor{B}\,\rho(\Gamma)\right\rfloor,
\label{eq:bf_def}
\end{equation}
so that fragile samples devote more of $B$ to routing-aware protection while stable ones default to the base signal. The anchor set $\mathcal{A}_\mathrm{anc}$ is then top-$B_f$ context positions in terms of $r(i)$:
\begin{equation}
\mathcal{A}_{\mathrm{anc}}
=
\mathrm{Top}_{B_f}\{r(i): i\in\{1,\ldots,L\}\}.
\label{eq:anchor_set}
\end{equation}
The anchors are reserved within each layer/head budget before applying the base scorer to the remaining positions. Thus, the retained context for layer $\ell$ and head $h$ , or $S_{\ell,h}$ is given by
\begin{equation}
S_{\ell,h}
=
\mathcal{A}_{\mathrm{anc}}
\cup
S^{\mathrm{base}}_{\ell,h}
\left(B-B_f;\{1,\ldots,L\}\setminus \mathcal{A}_{\mathrm{anc}}\right).
\label{eq:anchor_selection}
\end{equation}
where $S_{\ell,h}^{\mathrm{base}}(b;\mathcal{C})$ denotes the base scorer's selection from candidate set $\mathcal{C}$ under budget $b$. Since the total cache size remains $B$ per layer/head, MoEZip alters only the composition of the retained cache, not the budget itself.

\noindent\textbf{Query-agnostic variant.} MoEZip-Recon replaces the downstream instruction with the same fixed reconstruction query used by KVZip and computes routing-aware scores from the router and attention statistics of that reconstruction pass. It requires no additional model forward pass beyond KVZip's reconstruction pass, and the resulting compressed cache can be reused across downstream requests.
The original MoEZip instead conditions anchor selection on the
downstream instruction.

\noindent\textbf{Complexity.} MoEZip adds only a lightweight selection-time computation over the base scorer, while preserving the same retained-cache size and decoding procedure after eviction. Its wall-clock overhead is small, adding 0.88s on Qasper and 0.86s on average across 4K--32K NIAH contexts; details are in Appendix~\ref{appendix:complexity}.

%% file: experiment.tex
\begin{figure}[t!]
    \centering
    \includegraphics[width=1.0\linewidth]{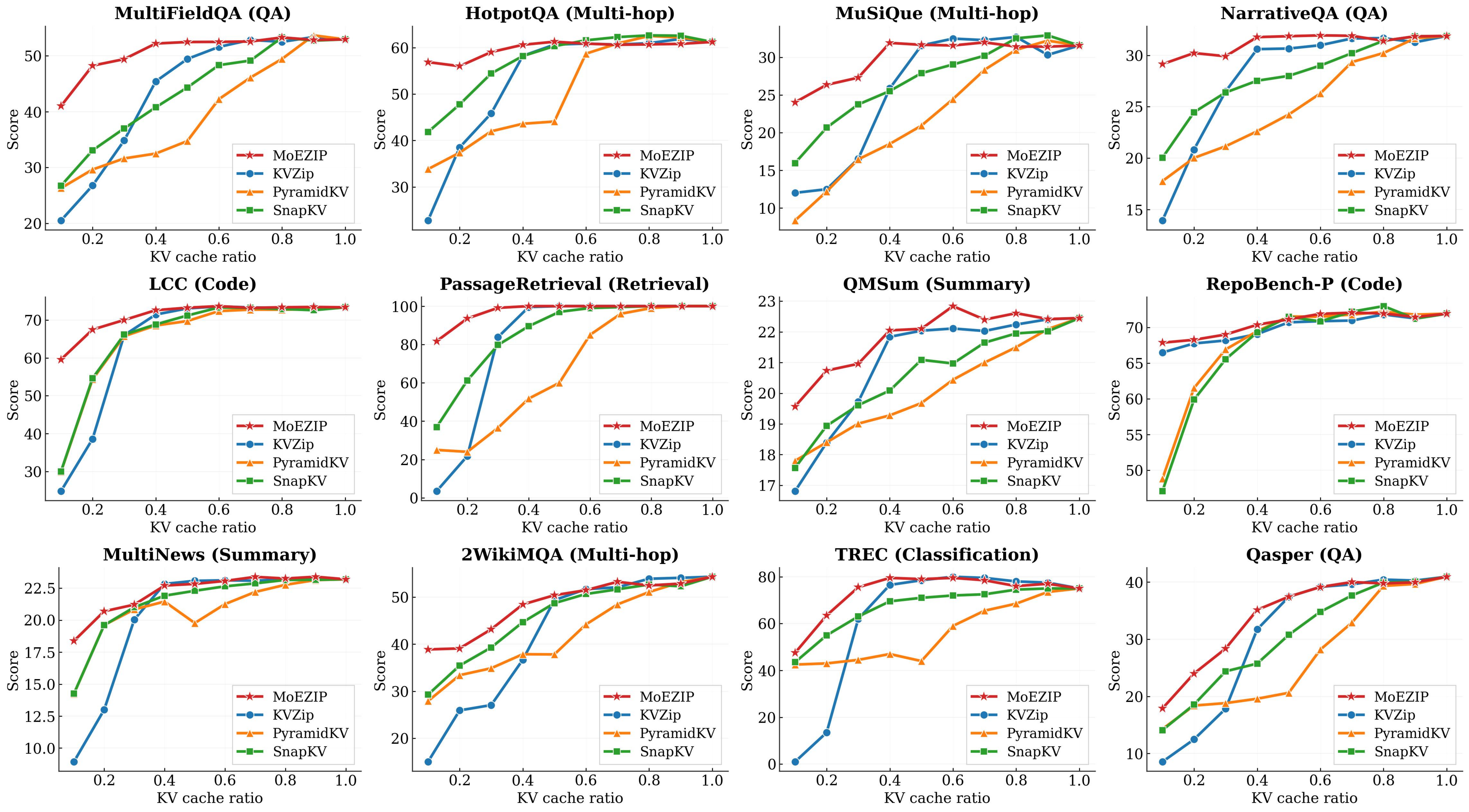}
    \caption{\textbf{Performance on 12 LongBench tasks across different KV cache ratios.} MoEZip consistently preserves higher task performance than existing KV compression baselines, especially under aggressive compression.}
    \label{fig:main_results}
\end{figure}

\section{Experiment}\label{sec:experiment}

\subsection{Experimental Settings}

\noindent\textbf{Baselines and Models.}
We compare MoEZip with KVZip~\cite{kim2025kvzip}, SnapKV~\cite{li2024snapkv}, and PyramidKV~\cite{cai2024pyramidkv}. We evaluate on Qwen3-30B-A3B-Instruct~\cite{yang2025qwen3} and GLM-4.7-Flash~\cite{zeng2025glm}, which use softmax and sigmoid routers, respectively.

\noindent\textbf{Evaluation Datasets.}
We evaluate on 12 LongBench tasks~\cite{bai2024longbench}, RULER~\cite{hsieh2024ruler}, and Needle in a Haystack (NIAH)~\cite{fu2024data} with context lengths from 4K to 32K. These benchmarks cover question answering, summarization, few-shot learning, synthetic retrieval, code completion, and diverse long-context tasks. We follow the official evaluation protocol for each benchmark and use the ROUGE-L F1 score for NIAH.

\noindent\textbf{KV Compression Setup.} We evaluate keep ratios from $0.1$ to $0.9$, matching retained-KV budgets, model checkpoints, prompt templates, and decoding settings across methods. MoEZip uses KVZip as its default base scorer; the SnapKV-based variant is evaluated separately. To control for query conditioning, we additionally compare KVZip with MoEZip-Recon using the same reconstruction query and reusable-cache setting. SnapKV and PyramidKV use their standard final-64-token observation window, which contains the complete template-defined downstream instruction in every evaluated example across the 12 LongBench tasks. This instruction span excludes task content embedded in raw question fields, such as passages and dialogues. Implementation details are provided in Appendix~\ref{appendix:implementation}.

\begin{figure}[t!]
    \centering
    \includegraphics[width=1.0\linewidth]{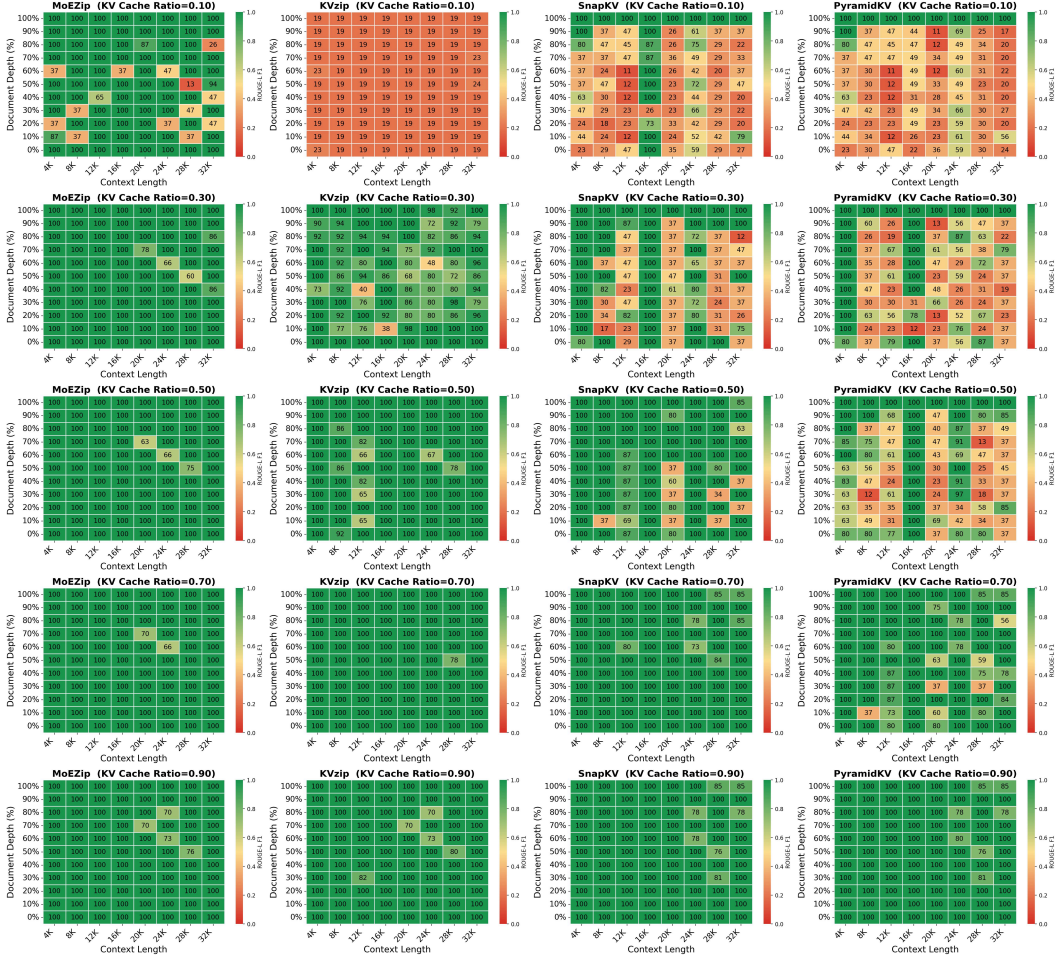}
    \caption{\textbf{Performance heatmaps over context length and document depth across different KV cache ratios.} We compare MoEZip, KVZip, SnapKV, and PyramidKV under KV cache keep ratios from 0.1 to 0.9.}
    \label{fig:niah}
\end{figure}

\subsection{Main Results}
\noindent\textbf{Results on LongBench.} Fig.~\ref{fig:main_results} shows that MoEZip achieves the best performance across most KV cache budgets.  In particular, MoEZip is more robust in the aggressive compression regime, where it consistently preserves higher task performance than KVZip~\cite{kim2025kvzip}, PyramidKV~\cite{cai2024pyramidkv}, and SnapKV~\cite{li2024snapkv}. 
\begin{wrapfigure}[9.5]{r}{0.42\textwidth}
    \vspace{-0.4cm}
    \centering
    \includegraphics[width=0.40\textwidth]{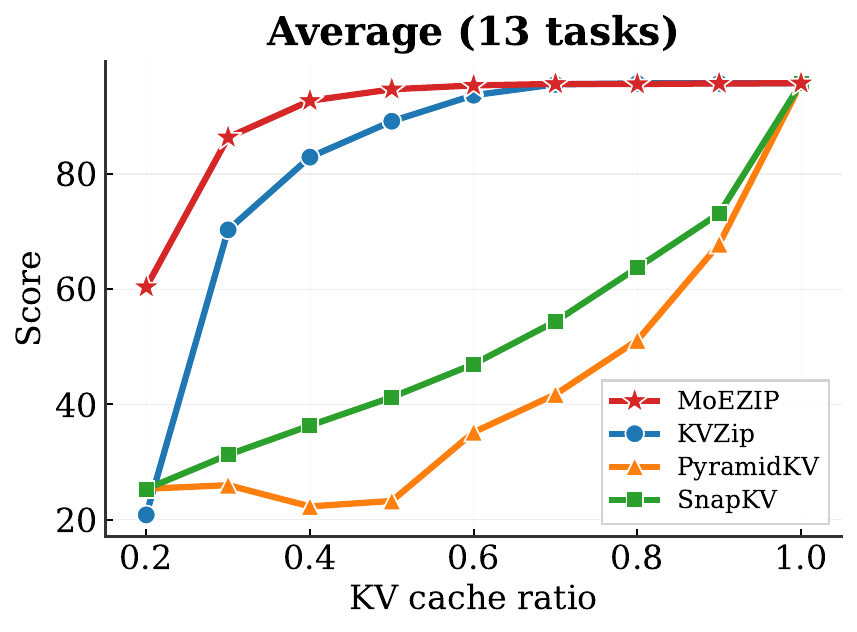}
    \caption{\textbf{Average performance on 13 RULER tasks across different KV cache ratios.}}
    \label{fig:ruler}
    \vspace{-0.6cm}
\end{wrapfigure}
The gains are most pronounced on reasoning- and retrieval-intensive benchmarks, while remaining stable on code and summarization tasks.  
Notably, under small KV cache ratios (i.e., when only a small fraction of the KV cache is retained), existing methods degrade substantially, with KVZip showing the most pronounced collapse and other baselines also suffering clear performance drops, whereas MoEZip remains comparatively robust and preserves strong performance across tasks.

\noindent\textbf{Results on NIAH.} Fig.~\ref{fig:niah} presents performance heatmaps over context length and document depth across KV cache ratios. Under light compression, most methods remain relatively strong, although MoEZip already shows the most uniform performance across the grid. The gap widens substantially in the high-compression regime: MoEZip consistently preserves strong performance across varying target positions and context sizes, while the baselines become much more sensitive to both factors. In particular, KVZip degrades most severely and eventually collapses almost entirely under aggressive compression. SnapKV and PyramidKV remain more competitive than KVZip, but still show substantial instability, especially when the relevant information is placed deeper in longer contexts. These results indicate that MoEZip is substantially more robust to both target position and context length under aggressive KV cache compression.

\begin{wrapfigure}[16]{r}{0.42\textwidth}
    \vspace{-0.5cm}
    \centering
    \includegraphics[width=0.40\textwidth]{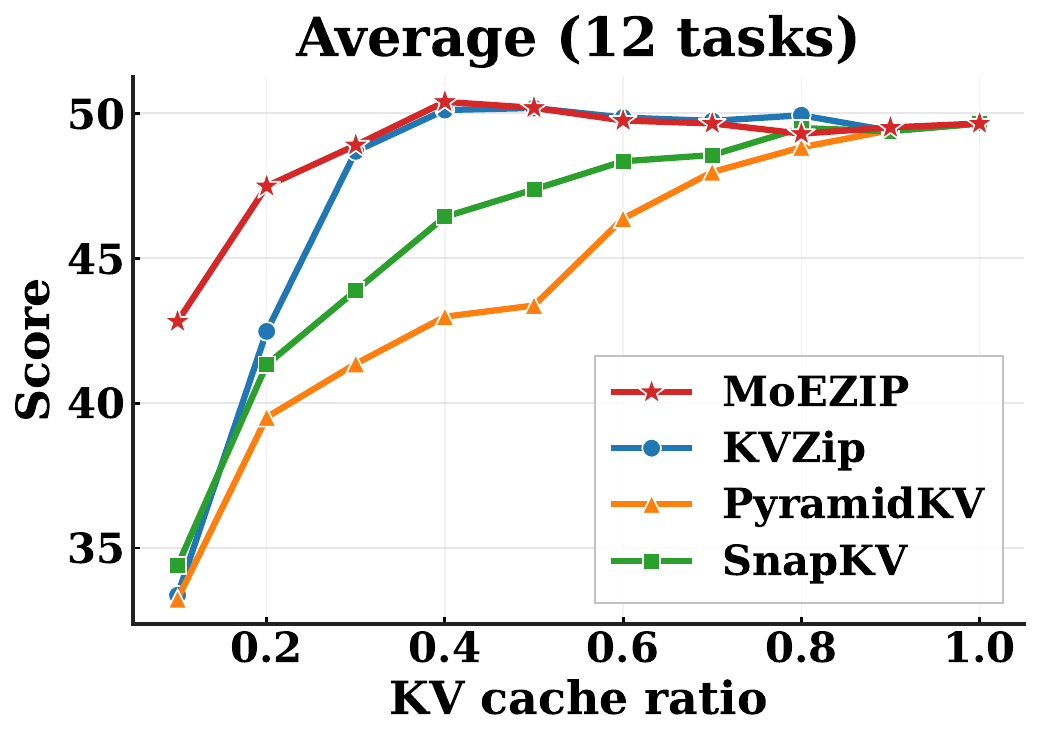}
    \vspace{-0.3cm}
    \caption{\textbf{Generalization to sigmoid-routed MoE models.} On GLM-4.7-Flash, MoEZip consistently achieves the best average performance across 12 LongBench tasks, with the largest gains under aggressive KV compression.}
    \label{fig:glm_comparison}
    \vspace{0.1cm}
\end{wrapfigure}

\noindent\textbf{Results on RULER.}  Fig.~\ref{fig:ruler} shows that MoEZip achieves the best average performance on RULER across KV cache ratios. The advantage is especially pronounced in the high-compression regime, where PyramidKV and SnapKV deteriorate substantially as the retained KV budget decreases, while KVZip, although relatively stronger, still remains consistently below MoEZip. Although all methods converge near the full-cache setting, MoEZip reaches this high-performance regime with a smaller KV budget, demonstrating stronger robustness under aggressive compression.

\noindent\textbf{Generalization across model scales and routers.} To examine generalization for sigmoid-routing models, we additionally evaluate MoEZip on GLM-4.7-Flash.  Fig.~\ref{fig:glm_comparison} shows that MoEZip remains consistently superior on average across 12 tasks for nearly all KV cache ratios. MoEZip also improves compression on Qwen3-235B-A22B-Instruct-2507~\cite{yang2025qwen3}: at keep ratio $0.1$, the 10-task LongBench macro increases from $27.6$ with KVZip to $44.7$. Detailed results are provided in Appendix~\ref{app:large-model}.

\subsection{Analyses}\label{sec:analyses}

\noindent\textbf{Effect of Routing-Fragility Scoring.} 
We evaluate the contribution of the routing fragility scoring based on Fisher-trace by comparing the full method with a variant that removes this component. As shown in Fig.~\ref{fig:additional_results_wo_fisher}, the full model consistently achieves better average performance, especially when the KV budget is small. This suggests that the Fisher-trace metric provides an effective estimate of routing fragility under compression, allowing the method to better preserve KV entries that are critical for stable expert routing. Without this component, performance degrades more noticeably under aggressive compression. Additional proxy comparisons and validation in Appendices~\ref{appendix:fisher_proxy_token}, \ref{appendix:fisher_proxy_sample}, ~\ref{appendix:surrogate_validation}, and \ref{appendix:worst} show that higher Fisher-trace bins correspond to higher expert-flip rates under compressed KV (\ref{appendix:fisher_proxy_token}), sample-level Fisher scores predict expert-flip rates (\ref{appendix:fisher_proxy_sample}), instruction-token Fisher predicts answer-token Fisher (\ref{appendix:surrogate_validation}), and Fisher-trace provides a stable default compared to worst-case curvature (\ref{appendix:worst}). 

\begin{wraptable}{r}{0.47\textwidth}
\centering
\small
\setlength{\tabcolsep}{3pt}
\vspace{-0.9cm}
\begin{tabular}{lccc}
\toprule
\textbf{Method}
& \shortstack{\textbf{Reusable}\\\textbf{cache}}
& \shortstack{\textbf{Keep}\\\textbf{0.2}}
& \shortstack{\textbf{Keep}\\\textbf{0.1}} \\
\midrule
KVZip        & Yes & 29.25 & 19.51 \\
MoEZip-Recon & Yes & 45.14 & 41.95 \\
MoEZip       & No  & \textbf{49.59} & \textbf{44.89} \\
\bottomrule
\end{tabular}
\caption{\textbf{Query conditioning and cache reuse.} LongBench 12-task macro scores on Qwen3-30B-A3B-Instruct-2507 at matched KV budgets. KVZip and MoEZip-Recon share a fixed
reconstruction query.}
\label{tab:query-agnostic}
\vspace{-0.5cm}
\end{wraptable}

\noindent\textbf{Effect of Adaptive Budget Allocation.} We study the role of adaptive budget allocation by comparing it against a static KV retention budget ($\lfloor B\cdot(\rho_\text{mid})\rfloor$), where $\rho_\text{mid}=(\rho_\text{min}+\rho_\text{max})/2$. Fig.~\ref{fig:additional_results_wo_budget} shows that removing adaptive budget allocation leads to a consistent drop in performance across KV cache ratios, with the gap again becoming more pronounced in the high-compression regime. This suggests that fragility-adaptive budgeting is more robust than fixed allocation under limited cache capacity.

\noindent\textbf{Effect of Head-wise aggregation.} We compare the proposed max pooling with a mean-pooling alternative. Fig.~\ref{fig:additional_results_mean_pooling} shows that replacing max pooling in the routing-aware score with mean pooling consistently hurts performance across nearly all KV cache ratios, with the largest gap appearing under aggressive compression. This suggests that the strongest routing-aware signal is more informative than the average signal.

\noindent\textbf{Compatibility with Different Base Scorers.} To examine whether the effectiveness of our approach depends on a particular base scoring rule, we apply it on top of SnapKV~\cite{li2024snapkv} and compare it against the original SnapKV baseline. Fig.~\ref{fig:additional_results_snapkv} shows that our variant consistently achieves higher average performance across 12 tasks over a wide range of KV cache ratios. The improvement is particularly notable under aggressive compression. These results indicate that our method is complementary to existing base scorers and can be integrated with them.

\noindent\textbf{Query-agnostic compression.} Table~\ref{tab:query-agnostic} separates the benefit of routing-aware selection from downstream query conditioning. At keep ratio $0.1$, MoEZip-Recon improves over KVZip by $22.44$ points while preserving cross-request cache reuse. It retains $88\%$ of the instruction-conditioned variant's improvement over KVZip.

\begin{figure*}[t!]
    \centering
    \begin{subfigure}[t]{0.25\textwidth}
        \centering
        \includegraphics[width=\linewidth]{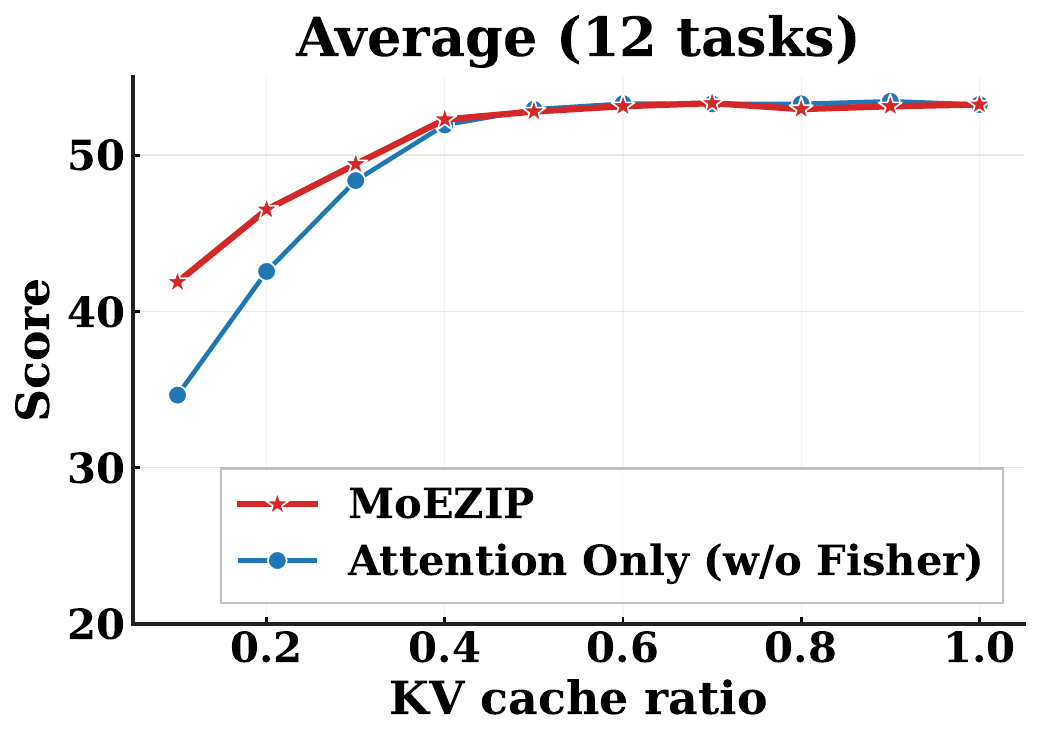}
        \caption{Without fragility score based on Fisher-trace.}
        \label{fig:additional_results_wo_fisher}
    \end{subfigure}\hfill
    \begin{subfigure}[t]{0.24\textwidth}
        \centering
        \includegraphics[width=\linewidth]{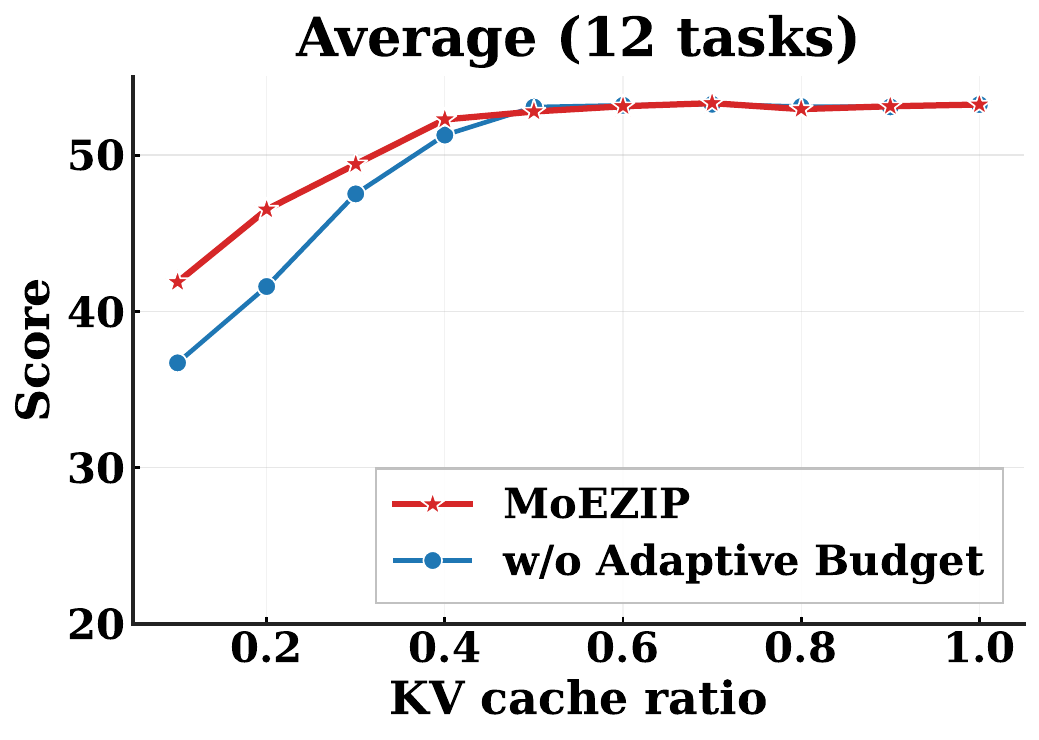}
        \caption{Using static budget allocation.}
        \label{fig:additional_results_wo_budget}
    \end{subfigure}\hfill
    \begin{subfigure}[t]{0.25\textwidth}
        \centering
        \includegraphics[width=\linewidth]{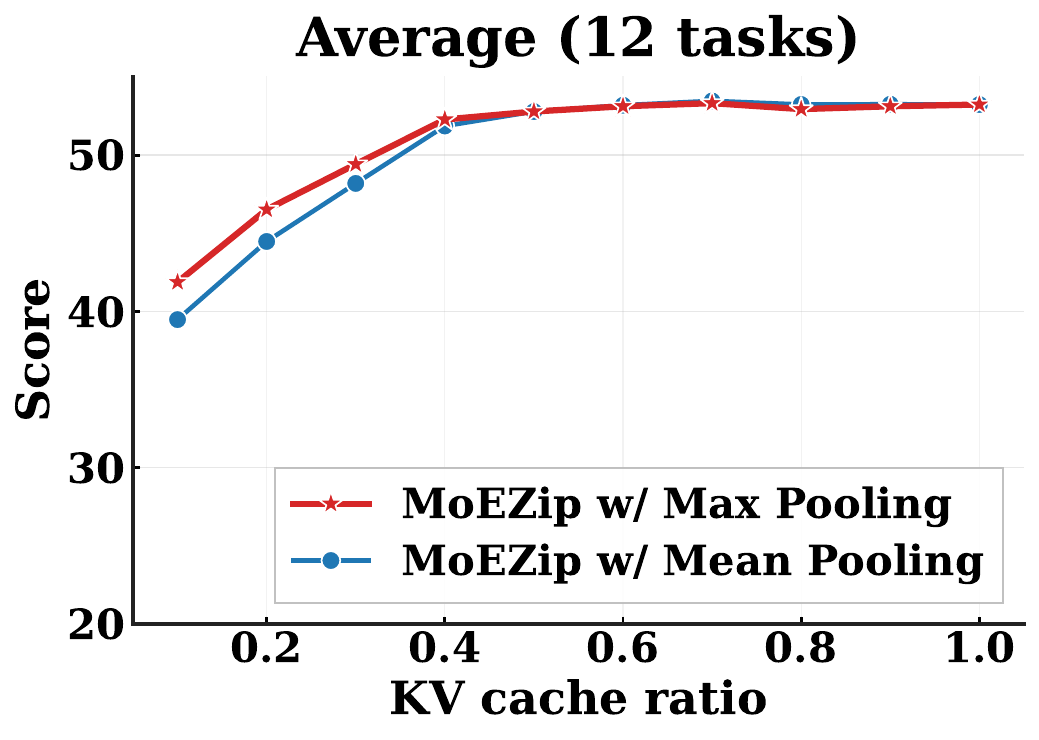}
        \caption{Head-wise aggregation variants for \eqref{eq:head_reduce}.}
        \label{fig:additional_results_mean_pooling}
    \end{subfigure}
    \begin{subfigure}[t]{0.24\textwidth}
        \centering
        \includegraphics[width=\linewidth]{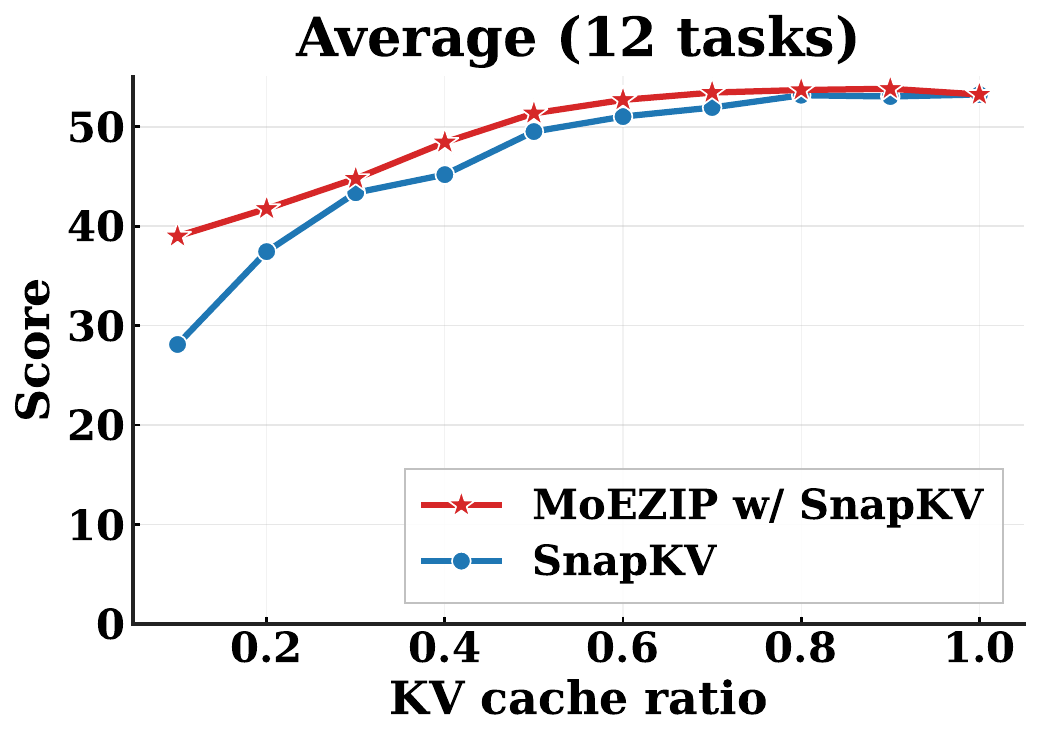}
        \caption{Using SnapKV as the base scorer.}
        \label{fig:additional_results_snapkv}
    \end{subfigure}\hfill
    \caption{\textbf{Average performance over 12 tasks across different KV cache ratios.} a)--(d) present ablations on the main components of MoEZip, including Fisher-based fragility modeling, adaptive budget allocation, head-wise aggregation, and the choice of base scorer. The full method consistently performs best, especially in the high-compression regime.}
    \label{fig:additional_results}
\end{figure*}

\subsection{Routing Intervention with a Fixed Retained Cache}
\label{sec:routing-intervention}

\noindent\textbf{Controlled intervention.} We evaluate Qwen3-30B-A3B at keep ratio $0.1$ while holding each method's retained context KV set fixed. At each instruction-token position and MoE layer, we replace compressed-run expert assignments with those from a full-KV reference. Expert computations still use activations from the intervened compressed-KV run; no hidden states or evicted KV entries are restored.

\noindent\textbf{Likelihood and task scores.}
Restoring instruction routing reduces KVZip's answer NLL by $0.70$, versus $0.27$ for MoEZip, narrowing their NLL gap from $0.88$ to $0.45$. The corresponding free-decoding LongBench gains are $10.2$ and $1.3$ points, respectively. Thus, routing changes affect answer quality even with the retained cache fixed, although correcting instruction routing does not recover all effects of KV eviction.

\begin{table}[t]
\centering
\small
\setlength{\tabcolsep}{4pt}
\begin{tabular}{llrrrr}
\toprule
\textbf{Metric} & \shortstack{\textbf{Retained}\\\textbf{set}}
& \textbf{Full KV} & \textbf{Compressed}
& \shortstack{\textbf{Oracle}\\\textbf{routing}} & $\boldsymbol{\Delta}$ \\
\midrule
NLL $\downarrow$ & KVZip  & 4.22 & 5.67 & 4.97 & $-0.70$ \\
                 & MoEZip & 4.22 & 4.79 & 4.52 & $-0.27$ \\
\midrule
LongBench $\uparrow$ & KVZip  & 53.7 & 19.5 & 29.7 & $+10.2$ \\
                     & MoEZip & 53.7 & 44.9 & 46.2 & $+1.3$ \\
\bottomrule
\end{tabular}
\label{tab:routing-intervention}
\end{table}

\noindent\textbf{Progressive routing correction.}
On the KVZip-retained cache across 12 LongBench tasks, we randomly correct $0\%$, $25\%$, $50\%$, $75\%$, or $100\%$ of instruction-routing disagreements between the full- and compressed-KV runs. Disagreement sites and all other inference conditions are fixed; intermediate conditions average three independently sampled subsets. For answer NLL $\mathcal{L}(c)$, define
\begin{equation}
\operatorname{Recovery}_{\mathrm{NLL}}(c)=
\frac{\mathcal{L}(0)-\mathcal{L}(c)}
     {\mathcal{L}(0)-\mathcal{L}(1)}.
\label{eq:progressive-recovery}
\end{equation}
NLL decreases monotonically with correction fraction: correcting $25\%$, $50\%$, and $75\%$ of sites recovers $23.8\%$, $56.7\%$, and $80.5\%$ of the improvement from complete correction. Full correction reduces NLL from $5.670$ to $4.970$, versus $4.220$ with full KV.

\begin{table}[t]
\centering
\small
\begin{tabular}{lcc}
\toprule
\textbf{Corrected sites} & \textbf{Answer NLL $\downarrow$}
& \textbf{Normalized recovery} \\
\midrule
0\%   & 5.670 & 0.0\% \\
25\%  & $5.503 \pm 0.015$ & 23.8\% \\
50\%  & $5.273 \pm 0.012$ & 56.7\% \\
75\%  & $5.107 \pm 0.008$ & 80.5\% \\
100\% & 4.970 & 100.0\% \\
\bottomrule
\end{tabular}
\caption{\textbf{Progressive routing correction.}
The retained KVZip cache is fixed; recovery is normalized between zero and complete correction.}
\label{tab:progressive-routing}
\end{table}

\noindent\textbf{Additional analyses} are presented in the Appendix~\ref{appendix:additional_experiments}.

%% file: related.tex
\section{Related Work}\label{sec:related}

\noindent\textbf{KV Cache Compressions.}
Recent studies have proposed a variety of KV cache eviction and compression methods to reduce the memory and latency cost of long-context LLM inference. StreamingLLM~\cite{xiao2023efficient} retains attention sinks and recent tokens to enable stable streaming inference, while H$_2$O~\cite{zhang2023h2o} formulates KV eviction around heavy hitters that dominate attention. SnapKV~\cite{li2024snapkv} identifies prompt positions likely to be important for future decoding from a short observation window, and PyramidKV~\cite{cai2024pyramidkv} exploits layer-wise differences in information concentration to assign non-uniform cache budgets across layers. AdaKV~\cite{feng2024ada} extends this line of work by adaptively reallocating budgets across attention heads, and DuoAttention~\cite{xiao2025duoattention} distinguishes retrieval heads from streaming heads to combine full and lightweight KV retention. KVZip~\cite{kim2025kvzip} departs from query-aware scoring and proposes query-agnostic compression via context reconstruction, improving cache reuse across multiple queries. However, these methods are primarily developed for dense LLMs. In MoE LLMs, KV eviction can change router outputs and destabilize expert assignments. MoEZip addresses this issue by explicitly preserving routing behavior during KV compression.

\noindent\textbf{Efficient LLM Architectures.}
Recent work has increasingly emphasized efficiency as a central design goal for long-context language models~\cite{li2025minimax, zeng2026glm}. Prior studies have explored this direction through alternative sequence modeling and efficient attention architectures, including linear attention methods such as DeltaNet~\cite{yang2024parallelizing}, Gated DeltaNet~\cite{yang2024gated}, and Lightning Attention~\cite{qin2024lightning}, as well as sparse attention methods such as BigBird~\cite{zaheer2020bigbird}, Native Sparse Attention~\cite{yuan-etal-2025-native}, and DSA~\cite{deepseekai2025deepseekv32}. These methods improve efficiency by modifying the underlying architecture or attention pattern, reflecting a broader shift toward scalable long-context modeling. In contrast, our work is complementary: rather than redesigning the model architecture, we focus on KV cache compression at inference time, with particular emphasis on preserving routing behavior in sparse MoE LLMs.

%% file: conclusion.tex
\section{Conclusion, Limitations, and Broader Impacts}\label{sec:conclusion}

In this paper, we revisited KV cache compression for MoE language models from the perspective of routing fragility. We showed that KV eviction can perturb expert routing and induce a MoE-specific source of quality degradation under aggressive compression. We proposed MoEZip, a routing-aware KV compression method that protects context positions important for stable expert routing. Across LongBench, NIAH, and RULER, MoEZip consistently outperforms strong baselines, especially in the high-compression regime.

A limitation of MoEZip is that it requires access to router outputs and instruction-to-context attention statistics, which may not be exposed in black-box inference settings.

As with other LLM efficiency methods, MoEZip may lower deployment costs for both beneficial applications and potential misuse. Overall, our results highlight routing fragility as an underexplored principle for KV compression in MoE models and suggest future opportunities for routing-aware compression in broader sparse architectures.

\section*{Acknowledgement}
This work was supported in part by the Institute of Information \& Communications Technology Planning \& Evaluation (IITP) grant funded by the Korea government (MSIT) (IITP-2026-RS-2020-II201819, ICT Creative Consilience Program, 10\%, and No. RS-2026-25507543, Development of AI Co-Scientist based on Scientific Causal World Model, 70\%); the National Research Foundation of Korea (NRF) grant funded by the Korea government (MSIT) (RS-2022-NR070834, RS-2026-25482946).

%% file: appendix.tex
\clearpage
\appendix
\section*{Appendix Contents}

\begin{itemize}
    \item \textbf{\ref{appendix:fisher} Local KL Expansion and Router-Specific Fisher Forms}
    \begin{itemize}
        \item \ref{appendix:router_fisher} Proof of Proposition~\ref{prop:fisher_trace}
        \item \ref{appendix:fisher_proxy_token} Token-Level Validation of the Fisher-Trace Proxy
        \item \ref{app:fisher-aggressive} Validation under Aggressive Compression
        \item \ref{appendix:fisher_proxy_sample} Sample-Level Validation of the Fisher-Trace Proxy
        \item \ref{appendix:surrogate_validation} Instruction-Token Surrogate Validation
        \item \ref{app:answer-oracle} Instruction-to-Answer Oracle Diagnostic
    \end{itemize}
    \item \textbf{\ref{appendix:worst} Worst-case Analysis}
    \item \textbf{\ref{appendix:implementation} Implementation Details}
    \item \textbf{\ref{appendix:complexity} Complexity}
    \item \textbf{\ref{appendix:additional_experiments} Additional Analyses}
    \begin{itemize}
        \item \ref{appendix:sensitivity_rho} Sensitivity Analysis of $\rho_{\min}$ and $\rho_{\max}$
        \item \ref{appendix:flip_analysis} Answer-Token Flip Rate per Compression Ratio
        \item \ref{appendix:layerwise_flip} Layer-Wise Router Instability under KV Compression
        \item \ref{appendix:query_flip} Router Flip Rate across Different Inputs for the Same Context
        \item \ref{appendix:CI} Paired Bootstrap Confidence Intervals
        \item \ref{app:query-agnostic} Query-Agnostic Compression and Cache Reuse
        \item \ref{app:anchor-granularity} Global versus Fine-Grained Anchor Selection
        \item \ref{app:anchor-overlap} Complementarity of Routing-Aware Anchors
        \item \ref{app:large-model} Evaluation on a 235B MoE Model
        \item \ref{app:kv-quantization} Composition with INT8 KV Quantization
        \item \ref{app:cross-layer-association} Cross-Layer Routing Associations and Anchor Examples
    \end{itemize}
\end{itemize}

\section{Local KL Expansion and Router-Specific Fisher Forms}\label{appendix:fisher}

\subsection{Proof of Proposition~\ref{prop:fisher_trace}}\label{appendix:router_fisher}
\begin{proof}
By the second-order expansion of KL divergence around a smooth parametric family \cite{amari2000methods},
\begin{equation}
D\!\left(p(\cdot|z)\,\|\,p(\cdot|z+\delta z)\right)
=
\frac{1}{2}(\delta z)^\top F(z)\,(\delta z)
+
o(\|\delta z\|^2),
\label{eq:app_kl_second_order}
\end{equation}
where
\begin{equation}
\left[F(z)\right]_{i,j}=
-\mathbb E_{p(\cdot\,|\,z)}
\!\left[
\frac{\partial^2}{\partial z_i\partial z_j}\log p(\cdot\,|\,z)
\right]
\label{eq:app_fisher_def}
\end{equation}
is the Fisher information matrix with respect to logit $z$. Alternatively,
\begin{equation}
F(z)=
-\mathbb E_{p(\cdot\,|\,z)}
\!\left[
\nabla^2_z \log p(\cdot\,|\,z)
\right]
\label{eq:app_fisher_def2}
\end{equation}
where $\nabla^2_z$ denotes the Hessian with respect to $z$.

We first show \eqref{eq:expected_kl_trace}. We have that
\begin{align}  
\mathbb{E}[D(p(\cdot|z)\,\|\, p(\cdot|z+\delta z))&]
\propto \mathbb{E}\left[(\delta z)^\top F(z)\,(\delta z)\right]\label{eq:s1}
\\&=\mathbb{E}\left[\text{trace}\left[(\delta z)^\top F(z)\,(\delta z)\right]\right]\label{eq:s2}
\\&=\mathbb{E}\left[\text{trace}\left[F(z)\,\delta z~\delta z^\top\right]\right]\label{eq:s3}
\\&=\text{trace}\left[F(z)\,\mathbb{E}\left[\delta z~\delta z^\top\right]\right]\label{eq:s4}
\\&=\text{trace}\left[F(z)\,(\Sigma_{\delta z} + \mu_{\delta z}\mu_{\delta z}^\top)\right]\label{eq:s5}
\\&=\text{trace}\left[F(z)\,\Sigma_{\delta z}\right] 
\\&=\sigma^2~\text{trace}\left[F(z)\right]\label{eq:s6}
\end{align}
where we used: for \eqref{eq:s1}, a scalar is equal to the trace of the scalar; for \eqref{eq:s2}, $\text{trace}(ABC) = \text{trace}(BCA)$, for \eqref{eq:s3}, linearity of expectations; for \eqref{eq:s5}--\eqref{eq:s6}, we have  \(
\mu_{\delta z} = \mathbf{0}\) and \(\Sigma_{\delta z}= \sigma^2 I\).

Next, we show \eqref{eq:fragility_softmax_closed} and \eqref{eq:fragility_sigmoid_closed}.
\paragraph{Softmax-type routing.}
From \eqref{eq:prelim_softmax_gate}, we have that
\[
\log p(x|z) =x^\top z - \log(\sum_e \exp(z_e))
\]
Thus,
\[
-\nabla_z \log p(x|z) =\text{softmax}(z) - x = g-x
\]
Thus, the negative Hessian of $p(x|z)$ with respect to $z$ is the Jacobian of $g=\text{softmax}(z)$ with respect to $z$, which is known to be $\text{diag}(g) - gg^\top$. Thus,
\begin{align}
F(z) = -\mathbb{E}[\nabla_z^2 \log p(x|z)] = \text{diag}(g) - gg^\top\label{eq:fim-softmax}    
\end{align}

Thus,
\[
\text{trace}[F(z)] = \text{trace}[\text{diag}(g)] - \text{trace}(gg^T) = \sum_e g_e - \sum_e g_e^2 = 1 - \sum_e g_e^2.
\]
which proves \eqref{eq:fragility_softmax_closed}.

\paragraph{Sigmoid-type routing.}
From \eqref{eq:prelim_sigmoid_gate}, we have that
\[
\log p(x|z) = \sum_e x_e \log g_e(z) + (1-x_e)\log (1-g_e(z))
\]
where $g_e(z) = \sigma(z_e)$. Then
\[
\frac{\partial}{\partial z_e} \log p(x|z) = x_e \frac{ g_e(z)(1-g_e(z))}{g_e(z)}  - (1-x_e) \frac{g_e(z)(1-g_e(z))}{1-g_e(z)} = x_e - g_e(z)
\]
Thus
\[
-\frac{\partial^2}{\partial z_e^2} \log p(x|z) = g_e(z)(1-g_e(z))
\]
and
\[
-\frac{\partial^2}{\partial z_e\partial z_{e'}} \log p(x|z) = 0
\]
yielding \[-\nabla_z^2 \log p(x|z) = \text{diag}( g\odot(1-g) )\] where $\odot$ denotes the elementwise product. Thus
\[
F(z) = \text{diag}( g\odot(1-g) )
\]
Thus, we have that
\[
\text{trace}(F(z)) = \sum_e g_e (1- g_e)
\]
which proves \eqref{eq:fragility_sigmoid_closed}.
\end{proof}

\paragraph{Interpretation of the Fisher trace.}
The quadratic KL expansion in Eq.~\eqref{eq:app_kl_second_order}
is direction-dependent, whereas its isotropic expectation is proportional
to the Fisher trace. The trace can also be written as the expected squared
norm of the score function:
\begin{equation}
\operatorname{trace}(F(z))
=\mathbb{E}_{p(\cdot|z)}
\left[\left\|\nabla_z\log p(x|z)\right\|_2^2\right].
\label{eq:app_score_norm}
\end{equation}
This motivates a direction-agnostic measure of local routing sensitivity
without requiring the eviction-induced perturbation direction.
The isotropic model is a local motivation, rather than an assumption
that actual KV-eviction shifts follow this distribution.
The router-specific trace expressions require $O(E)$ computation. Appendix~\ref{appendix:worst} considers worst-case curvature, while
Appendix~\ref{app:fisher-aggressive} evaluates the relationship between
the trace and realized flips under aggressive compression.

\subsection{Token-Level Validation of the Fisher-Trace Proxy}
\label{appendix:fisher_proxy_token}

To further validate the Fisher-trace score, we group $(\text{sample}, \text{decode step}, \text{layer})$ observations into quantiles according to their Fisher-trace values and measure the corresponding empirical expert-flip rate under KV compression. As shown in Fig.~\ref{fig:fisher_flip_bins}, higher Fisher-trace bins consistently exhibit higher flip rates across both HotpotQA and Qasper, and this separation becomes stronger under more aggressive compression. This provides direct empirical support that the Fisher trace captures meaningful routing fragility, beyond its role as a theoretically motivated proxy.

\begin{figure}[h!]
    \centering
    \includegraphics[width=0.9\linewidth]{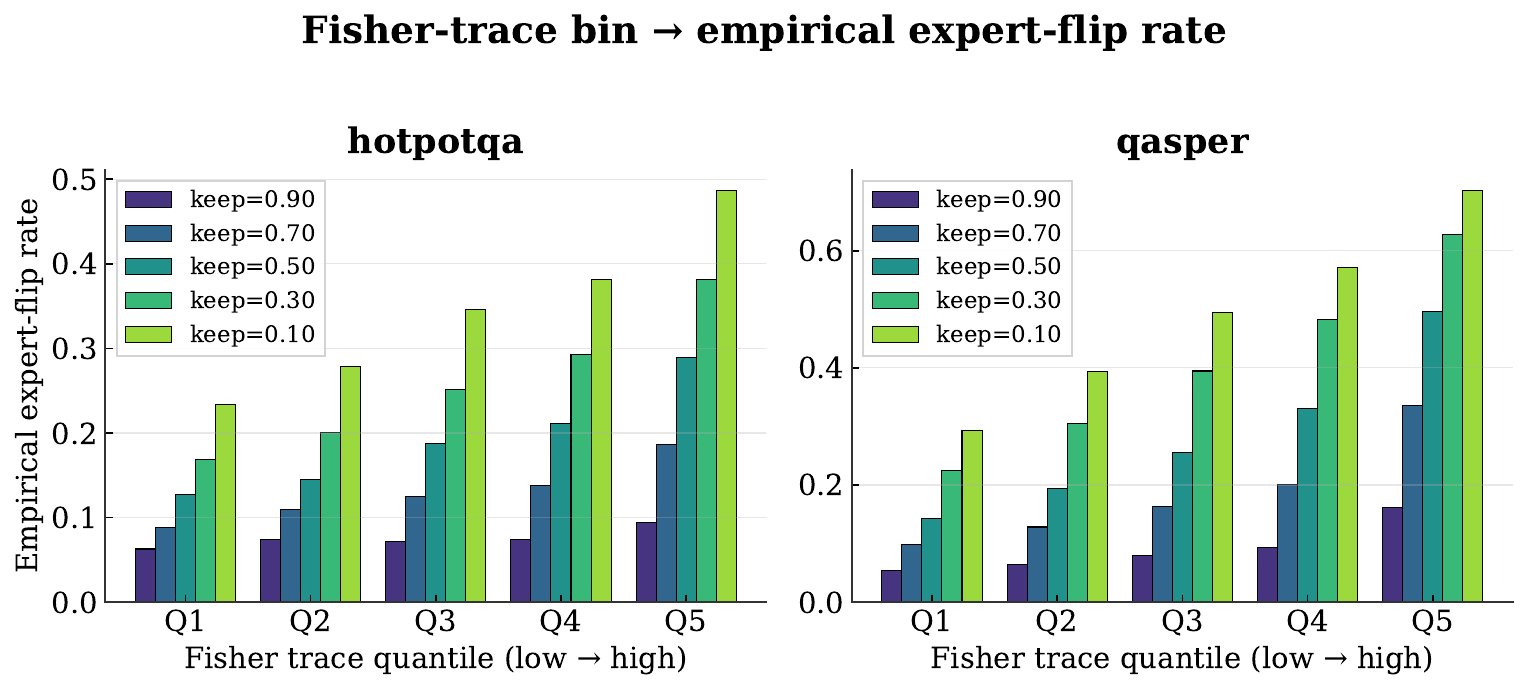}
    \caption{\textbf{Empirical expert-flip rate across Fisher-trace quantiles.} We bin $(\text{sample}, \text{decode step}, \text{layer})$ observations by Fisher trace from the full-KV router softmax and measure the mean top-1 expert-flip rate under compressed KV. Across both HotpotQA and Qasper, higher Fisher-trace bins consistently show higher flip rates, and the separation becomes larger as compression becomes more aggressive.}
    \label{fig:fisher_flip_bins}
\end{figure}

\subsubsection{Validation under Aggressive Compression}
\label{app:fisher-aggressive}
\label{appendix:fisher_aggressive}

We extend the token-level validation to 12 LongBench tasks using
Qwen3-30B with KVZip at keep ratio $0.1$. Fisher-trace quintiles are formed separately within each task and layer; we then aggregate realized expert-flip rates across 50 samples and evaluate decoding positions within each quintile.

\begin{table}[t]
\centering
\small
\setlength{\tabcolsep}{4pt}
\begin{tabular}{lrrrrrr}
\toprule
\textbf{Task} & \textbf{Q1} & \textbf{Q2} & \textbf{Q3}
& \textbf{Q4} & \textbf{Q5} & \textbf{Q5/Q1} \\
\midrule
MuSiQue     & .103 & .144 & .223 & .275 & .429 & 4.17 \\
2WikiMQA       & .085 & .127 & .168 & .210 & .340 & 4.00 \\
PsgRetrieval & .142 & .207 & .265 & .369 & .503 & 3.54 \\
RepoBench-P    & .185 & .302 & .375 & .475 & .609 & 3.29 \\
MultiFieldQA & .217 & .337 & .455 & .556 & .704 & 3.24 \\
Qasper       & .230 & .355 & .472 & .575 & .723 & 3.14 \\
HotpotQA     & .174 & .243 & .338 & .408 & .547 & 3.14 \\
NarrativeQA  & .199 & .296 & .400 & .481 & .608 & 3.06 \\
QMSUM        & .210 & .334 & .431 & .510 & .637 & 3.03 \\
TREC         & .202 & .260 & .340 & .454 & .604 & 2.99 \\
Multi-News   & .208 & .340 & .410 & .493 & .608 & 2.92 \\
LCC          & .331 & .447 & .527 & .629 & .757 & 2.29 \\
\bottomrule
\end{tabular}
\caption{\textbf{Realized expert-flip rates by Fisher-trace quintile.}
Qwen3-30B with KVZip at keep ratio $0.1$. Q1 and Q5 denote the lowest and highest quintiles. Ratios are computed from the displayed rounded flip rates.}
\label{tab:fisher-quintiles}
\end{table}

Flip rates increase monotonically across quintiles in all 12 tasks.
The highest quintile exhibits approximately $2.29$--$4.17$ times
the flip rate of the lowest quintile.
This supports the Fisher trace as a relative ranking signal under
aggressive compression across retrieval, question answering,
summarization, classification, and code tasks. It does not require interpreting the local perturbation approximation as an exact model of eviction-induced logit shifts.

\subsection{Sample-Level Validation of the Fisher-Trace Proxy} \label{appendix:fisher_proxy_sample}

We first perform a sample-level sanity check to verify whether the Fisher-trace fragility proxy is predictive of actual router instability. For each sample, we compute the mean Fisher trace over answer-token router states under the full KV cache, and correlate it with the sample-level expert flip rate induced by KV compression. We report both Spearman's rank correlation $\rho$ and Pearson correlation $r$ shown in the Table~\ref{tab:decode_fisher_flip_corr}.

\begin{table}[t!]
\centering
\small
\caption{\textbf{Sample-level correlation between answer-token Fisher trace and expert flip rate.} For each sample, we compute the mean Fisher trace from full-KV answer-token router states and correlate them with the expert flip rate under compressed KV. Both Spearman's $\rho$ and Pearson's $r$ are positive across all settings and increase under more aggressive compression, indicating that the Fisher-trace proxy captures sample-level routing fragility.}
\begin{tabular}{lccc}
\toprule
Dataset & Keep ratio & Spearman $\rho$ & Pearson $r$ \\
\midrule
HotpotQA & 0.7 & 0.26 & 0.28 \\
HotpotQA & 0.5 & 0.31 & 0.39 \\
HotpotQA & 0.3 & 0.50 & 0.54 \\
HotpotQA & 0.1 & 0.70 & 0.68 \\
Qasper   & 0.7 & 0.41 & 0.40 \\
Qasper   & 0.5 & 0.54 & 0.52 \\
Qasper   & 0.3 & 0.73 & 0.69 \\
Qasper   & 0.1 & 0.82 & 0.83 \\
\bottomrule
\end{tabular}
\label{tab:decode_fisher_flip_corr}
\end{table}

The correlations are positive in all settings and become substantially stronger as the keep ratio decreases. This trend is consistent across both HotpotQA and Qasper: under aggressive compression, samples with larger Fisher trace are much more likely to exhibit expert-routing flips. These results provide sample-level evidence that Fisher trace captures meaningful routing fragility rather than only acting as a local uncertainty heuristic.

This analysis validates the Fisher-trace proxy itself. Since answer-token router states are unavailable before generation, MoEZip uses instruction-token router states as a practical pre-generation surrogate; the effectiveness of this choice is further supported by the end-task improvements and routing-flip reductions reported in Sec.~\ref{sec:analyses} and Appendix~\ref{appendix:flip_analysis}.

\subsection{Instruction-token surrogate validation.}\label{appendix:surrogate_validation}

\begin{table}[t!]
\centering
\small
\caption{\textbf{Instruction-token surrogate validation.} $X_1$: pre-generation instruction-token mean Fisher score (observable at eviction time).
$X_2$: answer-token mean Fisher score under full KV (the quantity $X_1$ is meant to surrogate).
$Y$: compressed-KV expert flip rate (a downstream, highly aggregated statistic).
\textbf{Load-bearing claim} ($X_1 \leftrightarrow X_2$): the surrogate relation is significant across all tasks ($p<0.05$) and is independent of the keep ratio (both are computed under full KV).
\textbf{Supporting evidence} ($X_1, X_2 \leftrightarrow Y$): Fisher scores additionally correlate with the downstream flip rate, with strength varying by task structure.
Asterisks denote correlations significant at $p<0.05$.}

\begin{tabular}{lccccc}
\toprule
\multirow{2}{*}{Dataset} & \multirow{2}{*}{$\rho(X_1, X_2)$} & \multicolumn{2}{c}{$\rho(X_1, Y)$} & \multicolumn{2}{c}{$\rho(X_2, Y)$} \\
\cmidrule(lr){3-4} \cmidrule(lr){5-6}
 & & Keep 0.1 & Keep 0.3 & Keep 0.1 & Keep 0.3 \\
\midrule
QMSUM        & $+0.351^{*}$ & $+0.103$       & $+0.367^{*}$ & $+0.105$       & $+0.215$       \\
MultiFieldQA & $+0.475^{*}$ & $+0.351^{*}$   & $+0.242$     & $+0.825^{*}$   & $+0.563^{*}$   \\
\bottomrule
\end{tabular}
\label{tab:surrogate_validation}
\end{table}

The central design choice of MoEZip is to use instruction-token router states as a pre-generation surrogate for future answer-token routing fragility, since answer-token states are unavailable at eviction time. Table~\ref{tab:surrogate_validation} validates this choice. We separate the load-bearing surrogate claim from supporting end-to-end evidence, since the two play different roles in justifying MoEZip.

\paragraph{Load-bearing claim: surrogate validity ($X_1 \leftrightarrow X_2$).}
The relation that MoEZip actually depends on is whether instruction-token Fisher tracks answer-token Fisher under full KV: if it does, then the pre-generation signal $X_1$ is a faithful proxy for the routing fragility that ultimately governs generation quality. We find that $X_1$ is significantly correlated with $X_2$ on every evaluated task ($\rho = 0.351^{*}$ on QMSUM, $\rho = 0.475^{*}$ on MultiFieldQA, all $p<0.05$). This relation holds across both extractive QA and summarization, indicating that the surrogate is not tied to a specific task structure. Note that $\rho(X_1, X_2)$ is independent of the keep ratio because both quantities are computed from full-KV router states; only $\rho(\cdot, Y)$ varies with compression.

\paragraph{Supporting evidence: end-to-end link to the flip rate ($X_1, X_2 \leftrightarrow Y$).}
We additionally report correlations between Fisher scores and the compressed-KV flip rate $Y$. These are not necessary for the surrogate to be valid: $Y$ is a coarse, highly aggregated statistic, averaged over MoE layers, decode steps, and the particular eviction-induced perturbation realized at a given keep ratio, so a position-level Fisher score need not strongly predict it for the surrogate $X_1 \to X_2$ to be useful. Nevertheless, all reported correlations are positive, and on routing-localized tasks such as MultiFieldQA they are strong (e.g., $\rho(X_2, Y) = 0.825^{*}$ at keep $=0.1$). The weaker $X_1$--$Y$ correlation on QMSUM at keep $=0.1$ ($\rho = 0.103$) is consistent with summarization involving broad, distributed attention, so routing-critical positions are spatially diffuse; aggregating over many such positions averages out the per-position signal in $Y$. Importantly, the surrogate relation $X_1 \leftrightarrow X_2$ remains significant on QMSUM even at keep $=0.1$, so this reflects task-dependent diffuseness in $Y$ rather than a breakdown of the instruction-to-answer surrogate.

\paragraph{Takeaway.}
The load-bearing surrogate claim ($X_1 \to X_2$) holds uniformly across tasks and is what justifies using instruction-token router states for routing-aware KV selection. The end-to-end link to $Y$ is supporting evidence whose strength is task-dependent and strongest where routing perturbations are spatially localized, precisely the regime where position-level anchoring is most expected to help. The end-task improvements and routing-flip reductions reported in Section~\ref{sec:analyses} and Appendix~\ref{appendix:flip_analysis} provide complementary downstream validation across all evaluated tasks.

\subsubsection{Instruction-to-Answer Oracle Diagnostic}
\label{app:answer-oracle}
\label{appendix:answer_oracle}

\paragraph{Diagnostic design.}
We complement the score-level correlations above with a direct
task-performance comparison.
The answer-oracle condition replaces instruction-token positions with
teacher-forced reference-answer positions in both the Fisher and
token-to-context attention terms.
The retained-KV budget, routing-aware anchor budget $B_f$,
base scorer, and evaluation protocol are fixed across conditions.
Reference answers are used only to select oracle anchors and are
discarded before free decoding.
This is a privileged diagnostic rather than a deployable inference variant.

\paragraph{Recovery and anchor agreement.}
For task $d$, let $s_d^{\mathrm{base}}$,
$s_d^{\mathrm{inst}}$, and $s_d^{\mathrm{oracle}}$
denote the scores of KVZip, instruction-conditioned selection,
and answer-oracle selection.
We report
\begin{equation}
\begin{aligned}
R_d &=
\frac{s_d^{\mathrm{inst}}-s_d^{\mathrm{base}}}
     {s_d^{\mathrm{oracle}}-s_d^{\mathrm{base}}},\\
R_{\mathrm{agg}} &=
\frac{\sum_d(s_d^{\mathrm{inst}}-s_d^{\mathrm{base}})}
     {\sum_d(s_d^{\mathrm{oracle}}-s_d^{\mathrm{base}})}.
\end{aligned}
\label{eq:answer-oracle-recovery}
\end{equation}
Aggregate recovery is a ratio of summed improvements, not the arithmetic
mean of task-level recovery ratios.
For instruction and oracle anchor sets
$\mathcal{A}_d^{\mathrm{inst}}$ and $\mathcal{A}_d^{\mathrm{oracle}}$,
we measure Jaccard similarity as
\begin{equation}
J_d =
\frac{|\mathcal{A}_d^{\mathrm{inst}}\cap\mathcal{A}_d^{\mathrm{oracle}}|}
     {|\mathcal{A}_d^{\mathrm{inst}}\cup\mathcal{A}_d^{\mathrm{oracle}}|}.
\label{eq:answer-oracle-jaccard}
\end{equation}

\paragraph{LongBench results.}
Across four LongBench tasks at keep ratio $0.1$, selection based on
instruction tokens recovers approximately $82\%$ of
the oracle's aggregate improvement over KVZip.
Recovery is high on passage retrieval and HotpotQA ($97\%$ and $87\%$),
but lower on Qasper and QMSUM ($43\%$ and $20\%$), where anchor
agreement is also lower.
These results quantify task-dependent gaps that are not captured by
score-level correlation alone.

\begin{table}[t]
\centering
\small
\setlength{\tabcolsep}{4pt}
\begin{tabular}{lrrrrr}
\toprule
\textbf{Task} & \textbf{KVZip} & \textbf{Instruction}
& \shortstack{\textbf{Answer}\\\textbf{oracle}}
& \textbf{Recovery} & \textbf{Jaccard} \\
\midrule
PsgRetrieval & 3.5  & 63.3 & 65.1 & .97 & .96 \\
HotpotQA     & 22.8 & 48.6 & 52.3 & .87 & .88 \\
Qasper       & 8.6  & 18.3 & 31.2 & .43 & .71 \\
QMSUM        & 16.8 & 17.5 & 20.3 & .20 & .64 \\
\midrule
Macro / aggregate & 12.9 & 36.9 & 42.2 & .82 & .80 \\
\bottomrule
\end{tabular}
\caption{\textbf{Instruction-conditioned versus answer-oracle selection.}
Qwen3-30B at keep ratio $0.1$.
The final row reports macro task scores, aggregate recovery
from Eq.~\eqref{eq:answer-oracle-recovery}, and mean Jaccard similarity.}
\label{tab:oracle-longbench}
\end{table}

\paragraph{NIAH results.}
Across five RULER/NIAH settings with 4K-token contexts,
instruction-conditioned selection recovers approximately $89\%$ of
the oracle's aggregate improvement.
Recovery is $97\%$ across the three settings at keep ratio $0.2$
and $75\%$ across the two settings at keep ratio $0.1$.
The surrogate therefore retains most of the oracle benefit in these
retrieval settings, with a larger gap under more aggressive compression.

\begin{table}[t]
\centering
\small
\setlength{\tabcolsep}{3pt}
\begin{tabular}{lrrrrrr}
\toprule
\textbf{Task} & \textbf{Keep} & \textbf{KVZip}
& \textbf{Instruction} & \textbf{Oracle}
& \textbf{Recovery} & \textbf{Jaccard} \\
\midrule
NIAH-single-2   & .2 & 15.4 & 100.0 & 100.0 & 1.00 & .94 \\
NIAH-multikey-1 & .2 & 10.6 & 86.0  & 94.0  & .90  & .76 \\
NIAH-multiquery & .2 & 10.0 & 77.0  & 77.0  & 1.00 & .81 \\
NIAH-single-2   & .1 & 10.2 & 72.0  & 88.0  & .79  & .52 \\
NIAH-multikey-1 & .1 & 10.4 & 44.0  & 60.0  & .68  & .44 \\
\bottomrule
\end{tabular}
\caption{\textbf{Answer-oracle diagnostic on RULER/NIAH.}
All settings use 4K-token contexts.
Instruction and oracle conditions use matched retained-KV and anchor budgets.}
\label{tab:oracle-niah}
\end{table}

\paragraph{Signal-level check.}
On RULER QA-1 (SQuAD), instruction- and answer-conditioned context-position
scores have Spearman correlation $0.99$.
Their anchor sets have Jaccard similarity $0.94$ and answer-anchor
recall $0.97$. Recall is the fraction of oracle anchors also selected
by the instruction-based condition.
This provides a further signal-level check alongside task performance.

\section{Worst-case Analysis}\label{appendix:worst}
The worst-case perturbation corresponds to the case where perturbation $\delta z$ is in the direction of the eigenvector associated with the maximum eigenvalue of Fisher information matrix $F(z)$, or $\lambda_{\max}(F(z))$. We consider estimating $\lambda_{\max}(F(z))$ for two types of routing.

\paragraph{Softmax-type routing.} In \eqref{eq:fim-softmax}, we showed that $F(z)=\text{diag}(g)-gg^\top$. A closed-form $\lambda_{\max}$ is unavailable in this case, and we find an upper bound on $\lambda_{\max}$ with computational complexity $O(E)$ as follows.
\begin{proposition}
Let $g=\text{softmax}(z)$ and $F=\text{diag}(g)-gg^\top$. We have that
\begin{align}
\lambda_{\max}(F)  &\leq \min(B_1,B_2,B_3)\label{eq:softmax_worst_bound}\\
B_1 &= g_{\max}\\
B_2 &= \max_i[2g_{i}(1-g_{i})]\\
B_3 &= \sqrt{\sum_i g_i^2 - 2 \sum_i g_i^3 +\left(\sum_i g_i^2\right)^2}
\end{align}
where $g_{\max}$ denotes the maximum element of $g$.
\begin{proof}
For bound $B_1$, we use the triangular inequality for the maximum eigenvalue function as follows. For symmetric $A$ and $B$, we have
\[
\lambda_{\max}(A+B) \leq \lambda_{\max}(A) + \lambda_{\max}(B)
\]
Thus
\[
\lambda_{\max}(F) \leq \lambda_{\max}(\text{diag}(g)) + \lambda_{\max}(-gg^\top)
\]
Since $-gg^\top$ is rank-one matrix with one negative eigenvalue and $E-1$ zero eigenvalues, we have $\lambda_{\max}(-gg^\top)=0$. Since $\lambda_{\max}(\text{diag}(g)) = g_{\max}$, we have
\[
\lambda_{\max}(F) \leq g_{\max} = B_1
\]

For the bound $B_2$, we use Gershgorin circle theorem \cite{horn2012matrix}. The theorem states that, for any eigenvalue $\lambda$ of $F$, there  exists row index $i\in\{1,...,E\}$ such that
\begin{align}
F_{ii} - R_i\leq \lambda \leq F_{ii} + R_i\label{eq:gersh}
\end{align}
where $R_i$ is the sum of off-diagonal elements at row $i$, i.e.,
\[
R_i := \sum_{j\neq i} |F_{ij}|
\]
If we apply the theorem to $F$, we have that
\[
F_{ii} = g_{i}(1-g_{i})
\]
and
\[
R_i = |\sum_{j\neq i} F_{ij}| = |-\sum_{j\neq i} g_{i}g_j| = g_i \sum_{j\ne i} g_j = g_i(1-g_i)
\]
where we have used $g_i = 1- \sum_{j\ne i}g_j$. 
Thus, \eqref{eq:gersh} becomes
\[
0 \leq \lambda \leq 2g_i(1-g_i)
\]
This implies that
\[
\lambda_{\max} \leq \max_{i}[2g_i(1-g_i)] = B_2
\]
For bound $B_3$, we have in general
\[
\lambda_{\max}(F) \leq \| F\|_F
\]
where $\|\cdot\|_F$ denotes the Frobenius norm. Using its definition,
\[
\|F\|_F = \sqrt{\text{trace}(F^\top F)}
\]
We have that
\begin{align*}
 \text{trace}(F^\top F) &=\text{trace}\left[(\text{diag}(g) - gg^\top)(\text{diag}(g) - gg^\top)\right]
 \\&=\text{trace}[\text{diag}(g)^2] - 2\cdot\text{trace}[gg^\top \text{diag}(g)] + \text{trace}[(gg^\top)^2]    
\end{align*}
We have
\begin{align*}
\text{trace}[\text{diag}(g)^2] &= \sum_e g_e^2   
\\\text{trace}[gg^\top \text{diag}(g)] &= \text{trace}[g^\top \text{diag}(g) g] = g^\top \text{diag}(g) g = \sum_e g_e^3
\\\text{trace}[(gg^\top)^2] &= \text{trace}[g^\top(gg^\top)g] = g^\top(gg^\top)g = (g^\top g)^2 = \left(\sum_e g_e^2\right)^2
\end{align*}
Thus, we have $B_3 = \|F\|_F$, which completes the proof.
\end{proof}
\end{proposition}
The expression of the bounds $B_1,B_2$ and $B_3$ shows that all of them can be computed in $O(E)$ time.

\paragraph{Sigmoid-style routing.}
For sigmoid-style routing, Appendix~\ref{appendix:router_fisher} shows that
\[
F(z)
=
\operatorname{diag}(g\odot(1-g)).
\]
Because the Fisher matrix is diagonal, its maximum eigenvalue is available in closed form:
\begin{equation}
\lambda_{\max}\!\left(F\right)
=
\max_e g_e(1-g_e).
\label{eq:sigmoid_worst_exact}
\end{equation}
Thus, unlike the softmax case, the sigmoid worst-case curvature can be computed exactly in $O(E)$ time.

\begin{figure}[t!]
    \centering
    \includegraphics[width=1.0\linewidth]{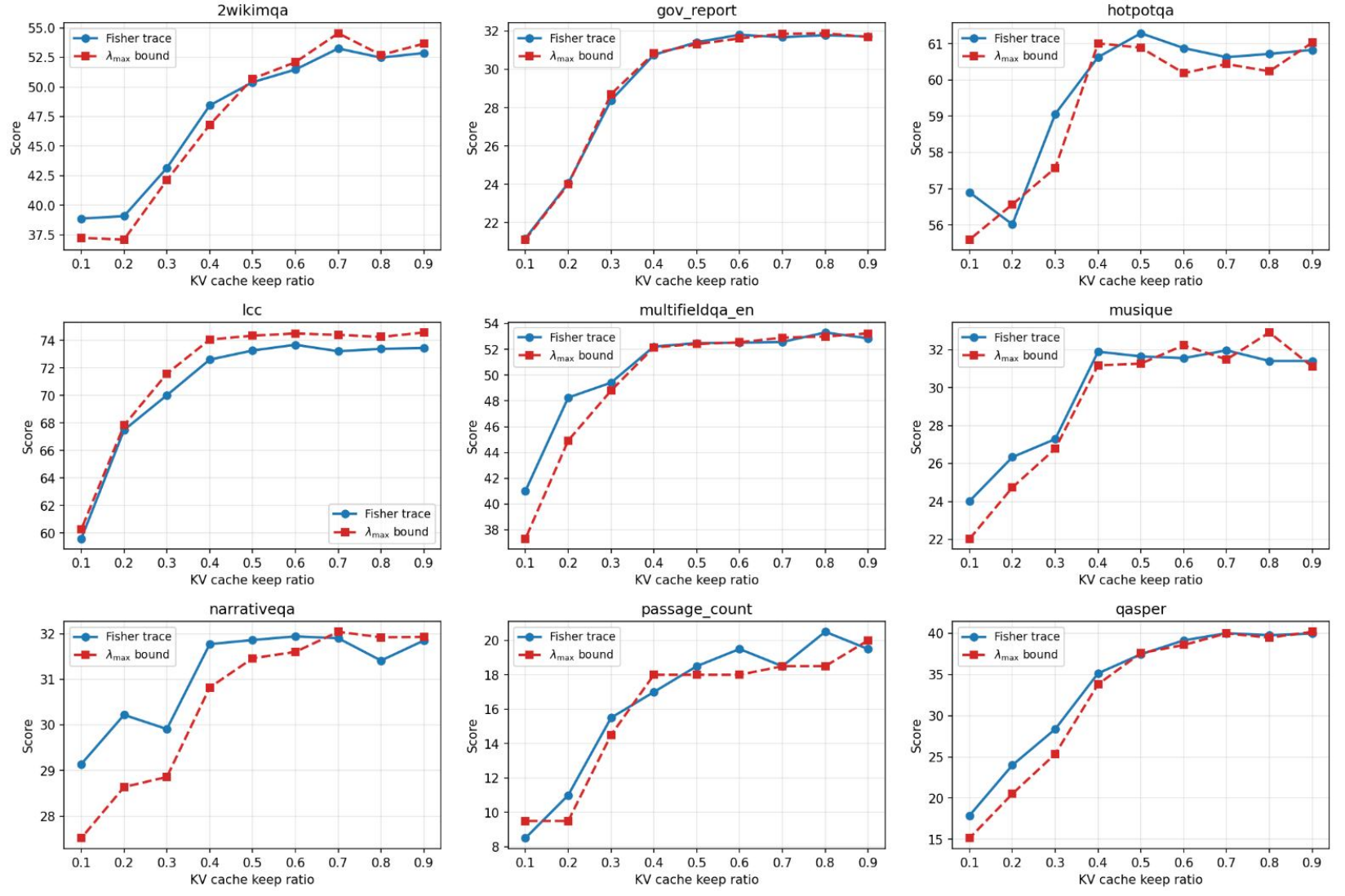}
    \caption{\textbf{Softmax-router worst-case bound versus Fisher trace.} The worst-case Fisher-bound score performs comparably to Fisher trace on several tasks, but does not consistently improve performance and can be less robust under aggressive compression. We therefore use Fisher trace as the default routing-fragility proxy.}
\label{fig:wost_softmax}
\end{figure}

\begin{figure}[t!]
    \centering
    \includegraphics[width=1.0\linewidth]{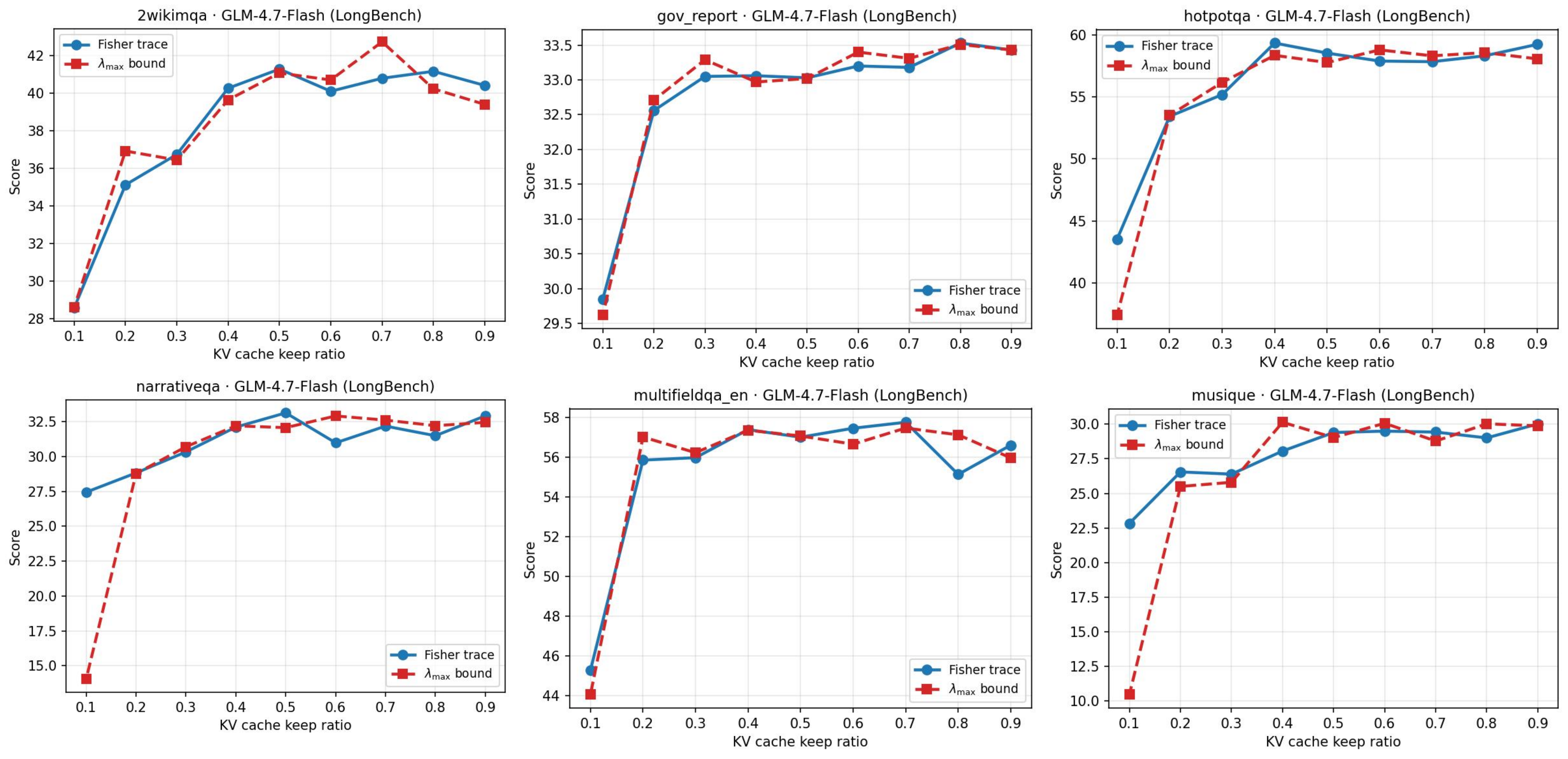}
    \caption{\textbf{Sigmoid-router worst-case score versus Fisher trace.} For sigmoid-style routing, the Fisher matrix is diagonal and the worst-case curvature is given exactly by $\max_e g_e(1-g_e)$. We compare this exact worst-case score with the Fisher trace on six LongBench tasks. The two scores show similar trends overall, but the worst-case score does not consistently outperform Fisher trace and can degrade more sharply in the high-compression regime.}
\label{fig:wost_sigmoid}
\end{figure}

\paragraph{Empirical comparison.}\label{appendix:worst-case_experiments} We compare the Fisher-trace score used in MoEZip with the worst-case curvature alternatives above. For softmax routing, we use the upper bound in Eq.~\eqref{eq:softmax_worst_bound}. For sigmoid-style routing, we use the exact value in Eq.~\eqref{eq:sigmoid_worst_exact}. Figures~\ref{fig:wost_softmax} and~\ref{fig:wost_sigmoid} report task performance across compression ratios.

Overall, the worst-case curvature scores are competitive with Fisher trace in several moderate-compression settings, indicating that worst-case Fisher geometry is also informative for routing fragility. However, they do not consistently outperform Fisher trace across tasks and can be less stable under aggressive compression. This is especially visible in the high-compression regime, where the bound-based score sometimes drops earlier or more sharply than Fisher trace. These results suggest that tighter worst-case curvature does not necessarily translate into better context-position ranking for KV retention.

We therefore use Fisher trace as the default score in MoEZip. Compared with worst-case curvature, Fisher trace provides a smoother direction-agnostic summary of the router state, admits simple closed forms for both softmax- and sigmoid-style routers, and empirically yields more stable performance across tasks and compression ratios.

\section{Implementation Details}
\label{appendix:implementation}

\paragraph{Algorithm.}
Algorithm~\ref{alg:pseudo} summarizes our router-fragility-aware prefix KV compression algorithm\footnote{Source code is provided in the Supplementary Material.}.

\begin{algorithm}[h!]
\caption{Router-Fragility-Aware Prefix KV Compression}
\label{alg:pseudo}
\begin{algorithmic}[1]
\Require Context prefix $x_{1:L}$, instruction tokens $y_{1:T}$, per-(layer, head) retention budget $B$
\Require Base scorer with retention signal $s^{\mathrm{base}}_{\ell,h}(i)$
\Ensure Compressed prefix KV cache
\State Run one prefill pass on $[x_{1:L}; y_{1:T}]$, caching keys/values and routing/attention statistics.
\For{each MoE layer $\ell \in \mathcal{M}$}
    \For{each instruction token $t = 1, \dots, T$}
        \State Obtain router gate $g_{t,\ell}$ and compute token-level routing fragility $f_{t,\ell}$ \Comment{Eq.~\eqref{eq:fragility_def}}
        \State Aggregate instruction-to-context attention $\bar{\alpha}_{t,\ell,i} \gets \max_h \alpha_{t,\ell,h,i}$ for all $i \in \{1,\dots,L\}$ \Comment{Eq.~\eqref{eq:head_reduce}}
    \EndFor
\EndFor
\State Compute context-side routing-importance scores $r(i) \gets \sum_{\ell \in \mathcal{M}} \sum_{t=1}^{T} \bar{\alpha}_{t,\ell,i}\, f_{t,\ell}$ for all $i \in \{1,\dots,L\}$ \Comment{Eq.~\eqref{eq:r_score}}
\State Compute sample-average routing fragility $\Gamma$ \Comment{Eq.~\eqref{eq:gamma_def}}
\State Compute adaptive anchor budget $B_f \gets \lfloor B \cdot \rho(\Gamma) \rfloor$ \Comment{Eq.~\eqref{eq:bf_def}}
\State Select routing-aware anchor set $\mathcal{A}_{\mathrm{anc}} \gets \operatorname{Top}_{B_f}\bigl\{r(i) : i \in \{1,\dots,L\}\bigr\}$
\For{each layer $\ell$ and attention head $h$}
    \State $\mathcal{S}_{\ell,h} \gets \mathcal{A}_{\mathrm{anc}} \cup \mathcal{S}^{\mathrm{base}}_{\ell,h}\bigl(B - B_f;\, \{1,\dots,L\} \setminus \mathcal{A}_{\mathrm{anc}}\bigr)$ \Comment{Eq.~\eqref{eq:anchor_selection}}
\EndFor
\State Evict prefix KV entries not in $\mathcal{S}_{\ell,h}$ for each $(\ell, h)$.
\end{algorithmic}
\end{algorithm}

\paragraph{Prefill-stage routing signals.} We instantiate MoEZip on top of an existing base KV compression method without changing its original masking or allocation rule. Given a context prefix $x_{1:L}$ and observed conditioning tokens $y_{1:T}$, we run a single prefill pass on $[x_{1:L}; y_{1:T}]$. Since the routing fragility of future decoding tokens is unavailable before generation, we treat $y_{1:T}$ as \emph{instruction routing tokens}. At each MoE layer, we collect their router gates and instruction-to-context attention weights.

For each instruction routing token $t$ and MoE layer $\ell$, we compute the token-level routing fragility score $f_{t,\ell}$ from the router gate using Eq.~\eqref{eq:fragility_def}. For softmax-routed models, this reduces to Eq.~\eqref{eq:fragility_softmax_closed}; for sigmoid-style models, we use Eq.~\eqref{eq:fragility_sigmoid_closed}. To aggregate instruction-to-context dependencies across attention heads, we use max pooling:
\begin{equation}
\bar{\alpha}_{t,\ell,i}
=
\max_h \alpha_{t,\ell,h,i}.
\end{equation}
We then compute the routing-importance score $r(i)$ in Eq.~\eqref{eq:r_score} and select the top-$B_f$ context positions as routing-aware anchors.

\paragraph{Adaptive anchor count.}
For adaptive budgeting, we compute the sample-level fragility ratio $\Gamma$ as the fraction of instruction-token router states whose fragility score exceeds the threshold $\tau$. We convert $\Gamma$ into an anchor ratio by
\begin{equation}
\rho(\Gamma)
=
\rho_{\min}
+
(\rho_{\max}-\rho_{\min})\Gamma.
\end{equation}
Given the current retention budget $B$ of the base scorer, we set the number of routing-aware anchors to
\begin{equation}
B_f
=
\operatorname{clip}
\!\left(
\left\lfloor B\,\rho(\Gamma)\right\rfloor,\;
B_{\min},\;
B_{\max}
\right).
\end{equation}
The clipping bounds are not task-specific tuning parameters; they only prevent the anchor set from becoming empty under tight budgets or overly dominant under large budgets. Unless otherwise stated, we use $\rho_{\min}=0.05$, $\rho_{\max}=0.25$, $B_{\min}=32$, and $B_{\max}=128$ 
in all experiments. We set the threshold to 85\% of the theoretical maximum Fisher trace: \(\tau=0.85(1-1/E)\) for softmax routing and \(\tau=0.85E/4\) for sigmoid routing, where \(E\) denotes the number of routed experts. We emphasize that $\tau$ is used only to estimate the sample-level fragility ratio $\Gamma$ in Eq.~(\ref{eq:gamma_def}); context-position ranking is still governed by the continuous routing-fragility score in Eq.~(\ref{eq:fragility_def}). Therefore,
MoEZip is not driven by hard thresholding of individual context positions. Sensitivity to $\rho_{\min}$ and $\rho_{\max}$ is analyzed in
Appendix~\ref{appendix:sensitivity_rho}.

\paragraph{Evaluation metrics.}
For LongBench~\cite{bai2024longbench}, we follow the official evaluation protocol and use the standard
metric associated with each task: F1 for question-answering tasks, ROUGE-L for summarization tasks, accuracy for classification and few-shot learning tasks, and edit similarity for code-completion tasks. For RULER~\cite{hsieh2024ruler}, we use the official evaluation scripts and report the average score over its 13 tasks. For NIAH~\cite{fu2024data}, we report Rouge-L F1 between the predicted answer and the target needle answer, treating needle retrieval as an answer-extraction task.

\paragraph{Compute resources.}\label{appendix:resource} All experiments were conducted on NVIDIA A100 GPUs with 80GB memory. Each evaluation run used two A100 GPUs, except for the 235B-model experiments described in Appendix~\ref{app:large-model}, which used eight. The reported experiments, including MoEZip and all baselines, required approximately 15{,}000--16{,}000 A100 GPU-hours in total. This estimate excludes preliminary exploratory or failed runs.

\begin{table}[h!]
    \centering
    \caption{\textbf{Scoring-time overhead of MoEZip on Qasper.} We compare the average scoring time of the base KVZip scorer and MoEZip under the same evaluation setting. The additional overhead of Fisher-based scoring remains modest.}
    \begin{tabular}{c|c|c}
        Method & Avg. Scoring Time (s) & Additional Time (s) \\ \hline
        KVZip (base) & 9.36 & - \\
        MoEZip (+Fisher score) & 10.24 & 0.88 \\
    \end{tabular}
    \label{tab:inference_qasper}
\end{table}

\begin{figure}[h!]
    \centering
    \includegraphics[width=0.6\linewidth]{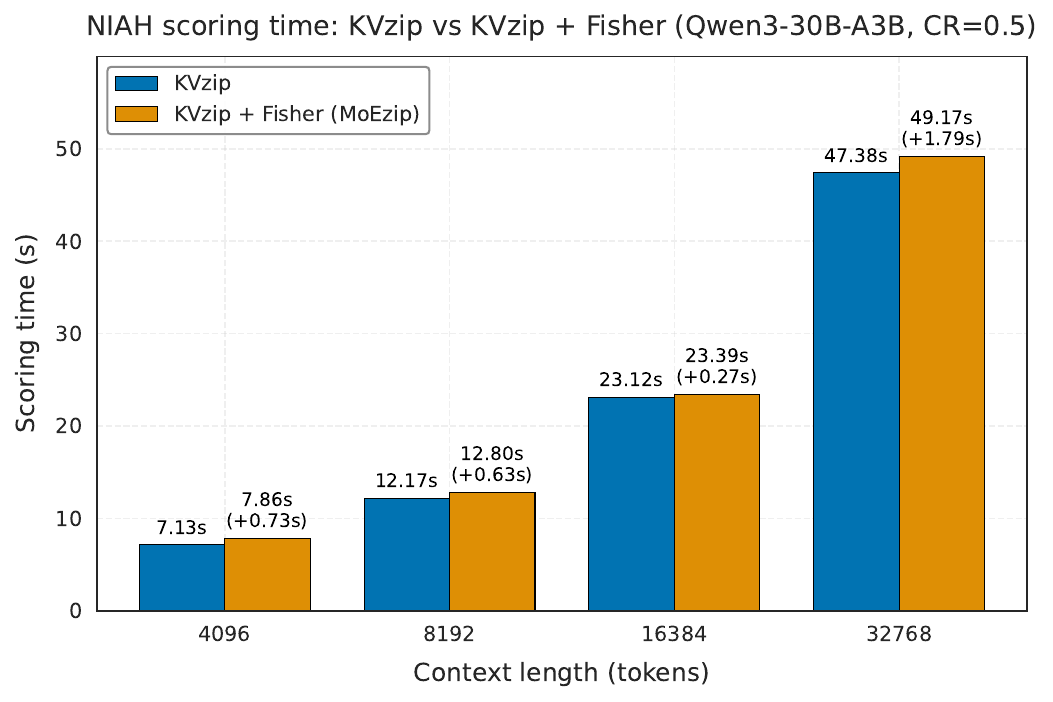}
    \caption{\textbf{Scoring-time overhead on NIAH.} Adding the proposed Fisher-based routing refinement to KVZip introduces only modest additional scoring cost on Qwen3-30B-A3B at a KV cache ratio of 0.5, with absolute overhead ranging from 0.27s to 1.79s across context lengths.}
    
    \label{fig:inference_niah}
\end{figure}

\paragraph{Query conditioning and cache reuse.}
\label{appendix:query_conditioning}
MoEZip-Recon uses the same fixed reconstruction query as KVZip.
It reuses router and attention statistics from that reconstruction pass
to compute routing-aware anchors, requiring no additional model forward
pass beyond KVZip's reconstruction pass.
Its final compressed cache is independent of downstream instructions
and can therefore be reused across requests.
The original MoEZip instead uses the downstream instruction for
anchor selection. A matched comparison is reported in
Appendix~\ref{app:query-agnostic}.

\paragraph{Observation-window fairness.}
MoEZip uses the template-defined downstream instruction span.
SnapKV and PyramidKV use their standard final-64-token observation window.
After applying the evaluated prompt templates and tokenizer, the complete
instruction span is contained within this window for every example
across the 12 LongBench tasks; the maximum instruction length is 64 tokens.
For this check, the downstream instruction is distinguished from task
content embedded in raw question fields, such as passages, dialogues,
and retrieval abstracts.
The KVZip--MoEZip-Recon comparison additionally matches the reconstruction
query, base scorer, retained-KV budget, and reusable-cache setting.

\section{Complexity}
\label{appendix:complexity}

Our method introduces only a modest additional overhead over the base scoring procedure. The extra computation comes from: (i) computing instruction-token router fragility scores, (ii) aggregating instruction-to-context attention, (iii) forming the position-wise refinement score in Eq.~\eqref{eq:r_score}, and (iv) incorporating this refinement into the base score before selection.

Let $|\mathcal{I}|$ denote the number of instruction tokens, $L$ the prefix length, $E$ the number of experts, and $\mathcal M$ the set of MoE layers used for refinement. The additional cost of router fragility computation is $O(|\mathcal M||\mathcal{I}|E)$, and the attention-based aggregation over context positions costs $O(|\mathcal M||\mathcal{I}|L)$, giving a total overhead of
\begin{equation}
O\!\bigl(|\mathcal M||\mathcal{I}|(E+L)\bigr).
\end{equation}
In long-context settings where $L \gg E$, this is dominated by $O(|\mathcal M||\mathcal{I}|L)$.

In practice, the overhead is small because our method reuses router outputs and attention statistics from the scoring stage and does not require any additional transformer forward pass. Since $|\mathcal{I}| \ll L$ in typical tasks, the added cost is modest relative to the overall prefill cost. 

We measure the overhead of Fisher-based scoring for MoEZip in Table~\ref{tab:inference_qasper} and Fig.~\ref{fig:inference_niah}.

\section{Additional Analyses}\label{appendix:additional_experiments}

\begin{figure}[h!]
    \centering
    \includegraphics[width=1.0\linewidth]{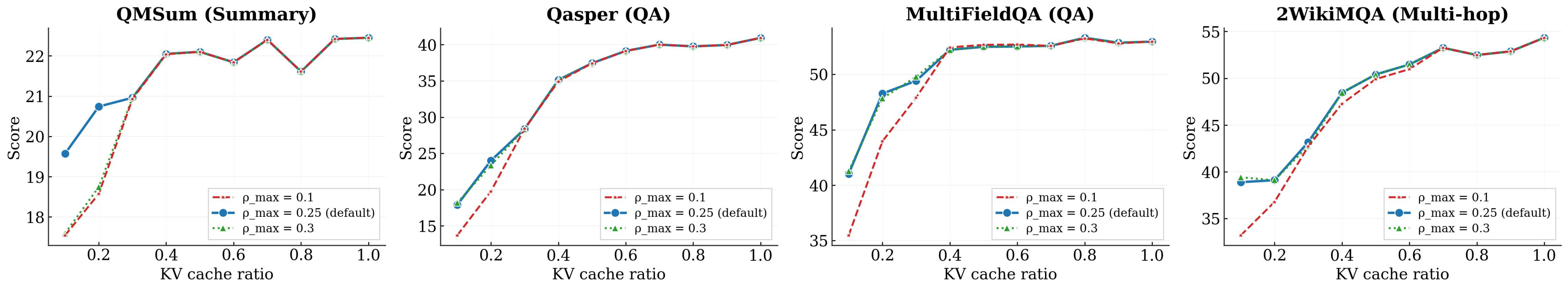}
    \caption{\textbf{Sensitivity analysis of $\rho_{\max}$ on representative LongBench tasks.} We vary the maximum anchor ratio $\rho_{\max}\in\{0.1,0.25,0.3\}$ while fixing $\rho_{\min}=0.05$. Across QMSum, Qasper, MultiFieldQA, and 2WikiMQA, the overall performance is relatively stable over a broad range of $\rho_{\max}$ values. The default setting $\rho_{\max}=0.25$ consistently provides strong or near-best performance, while overly small values can slightly weaken robustness in the high-compression regime by limiting the number of route-aware anchors.}
    \label{fig:sensitivity_rho_max}
\end{figure}

\begin{figure}[h!]
    \centering
    \includegraphics[width=1.0\linewidth]{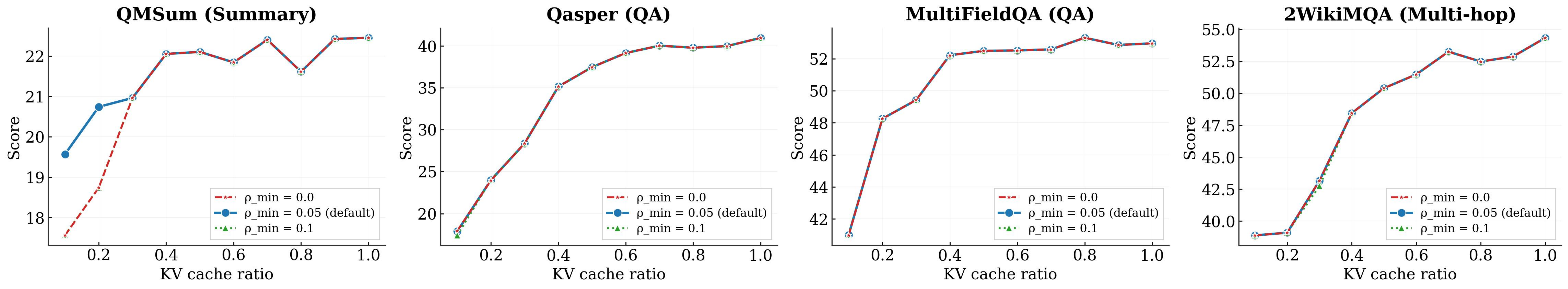}
    \caption{\textbf{Sensitivity analysis of $\rho_{\min}$ on representative LongBench tasks.} We vary the minimum anchor ratio $\rho_{\min}\in\{0.0,0.05,0.1\}$ while fixing $\rho_{\max}=0.25$. The results remain highly stable across all four tasks, indicating that MoEZip is not overly sensitive to the exact choice of $\rho_{\min}$. The default setting $\rho_{\min}=0.05$ achieves consistently strong performance, while different values lead only to minor variations across KV cache ratios.}
    \label{fig:sensitivity_rho_min}
\end{figure}

\subsection{Sensitivity Analysis of $\rho_{\min}$ and $\rho_{\max}$.}
\label{appendix:sensitivity_rho}

We study the sensitivity of MoEZip to the adaptive anchor-budget parameters $\rho_{\min}$ and $\rho_{\max}$ on four representative LongBench tasks: QMSum, Qasper, MultiFieldQA, and 2WikiMQA. Recall that the anchor ratio is defined as
\[
\rho(\Gamma)=\rho_{\min}+(\rho_{\max}-\rho_{\min})\Gamma,
\]
where $\Gamma$ denotes the sample-level routing fragility ratio. Thus, $\rho_{\min}$ controls the minimum degree of route-aware protection even for relatively stable samples, while $\rho_{\max}$ determines the maximum fraction of anchors allocated to highly fragile samples.

Fig.~\ref{fig:sensitivity_rho_max} varies $\rho_{\max}$ while fixing $\rho_{\min}=0.05$. Overall, the method is fairly robust across different choices of $\rho_{\max}$, and the default setting $\rho_{\max}=0.25$ consistently achieves strong or near-best performance. When $\rho_{\max}$ is set too small, the method can become slightly less robust under aggressive compression, since the route-aware component is prevented from allocating enough anchors to highly fragile cases. Increasing $\rho_{\max}$ beyond the default yields only marginal changes, suggesting that the method does not require precise tuning of the upper anchor budget.

Fig.~\ref{fig:sensitivity_rho_min} varies $\rho_{\min}$ while fixing $\rho_{\max}=0.25$. The results are even more stable across different choices of $\rho_{\min}$, with only minor variations observed over KV cache ratios and tasks. This indicates that MoEZip is not overly sensitive to the exact amount of minimum route-aware protection, as long as a modest baseline level of anchors is preserved.

Overall, these results show that MoEZip is robust to the choice of adaptive budget parameters. The default setting $(\rho_{\min},\rho_{\max})=(0.05,0.25)$ provides a good balance across tasks without requiring careful per-task tuning.

\begin{figure}[h!]
    \centering
    \includegraphics[width=1.0\linewidth]{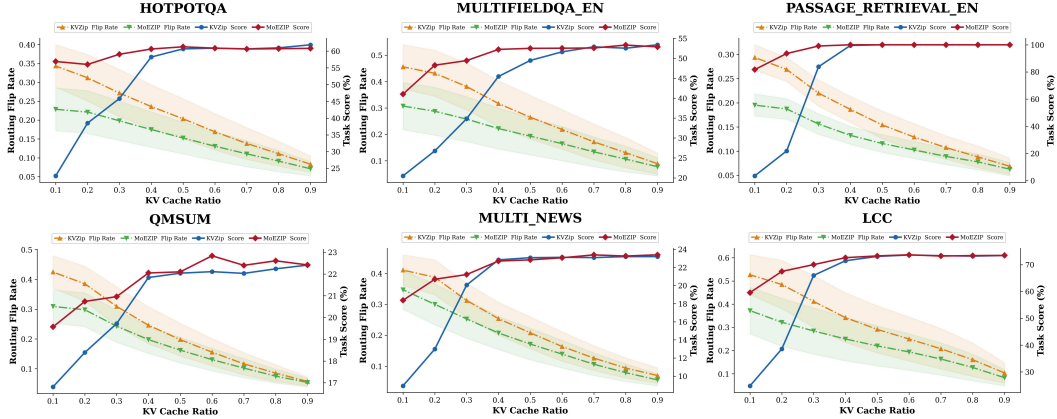}
    \caption{\textbf{Task quality and router flip rate across compression ratios.} We compare task performance and expert flip rate of answer-tokens under increasing KV compression. Across all six tasks, MoEZip consistently maintains lower flip rates than KVZip. The shaded region marks the compression range where KVZip begins to exhibit a sharp quality drop, while MoEZip remains notably more stable.}
    \label{fig:quality_cliff}
\end{figure}

\begin{figure}[t!]
    \centering
    \begin{minipage}[t]{0.48\linewidth}
        \centering
        \includegraphics[width=\linewidth]{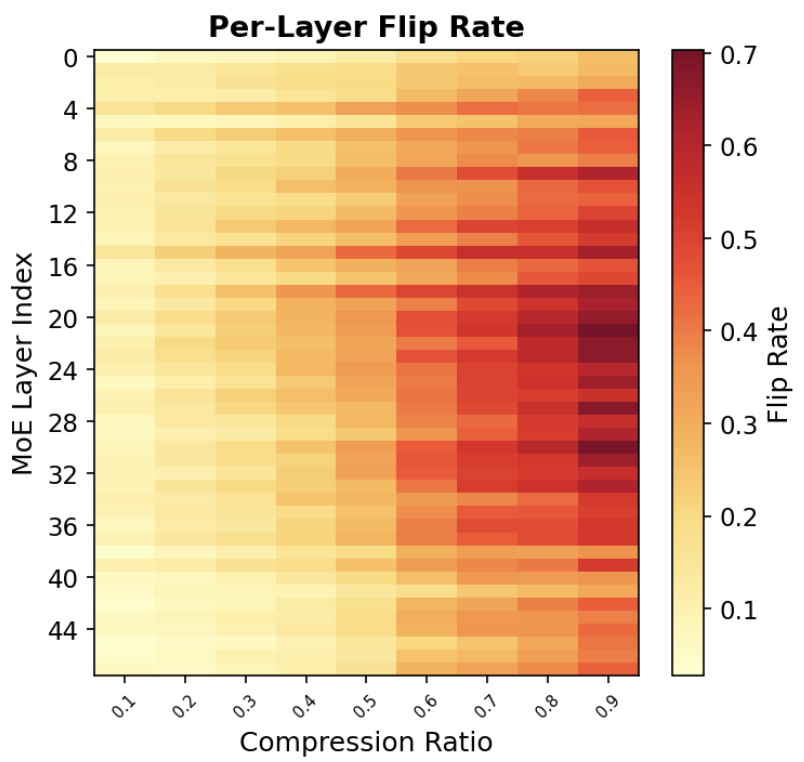}
        \caption{\textbf{Layer-wise router flip rate of KVZip across compression ratios, averaged over Qasper and HotpotQA.} Router instability increases steadily with compression and becomes most pronounced in the middle-to-late MoE layers.}
        \label{fig:layer_flip_kvzip}
    \end{minipage}\hfill
    \begin{minipage}[t]{0.48\linewidth}
        \centering
        \includegraphics[width=\linewidth]{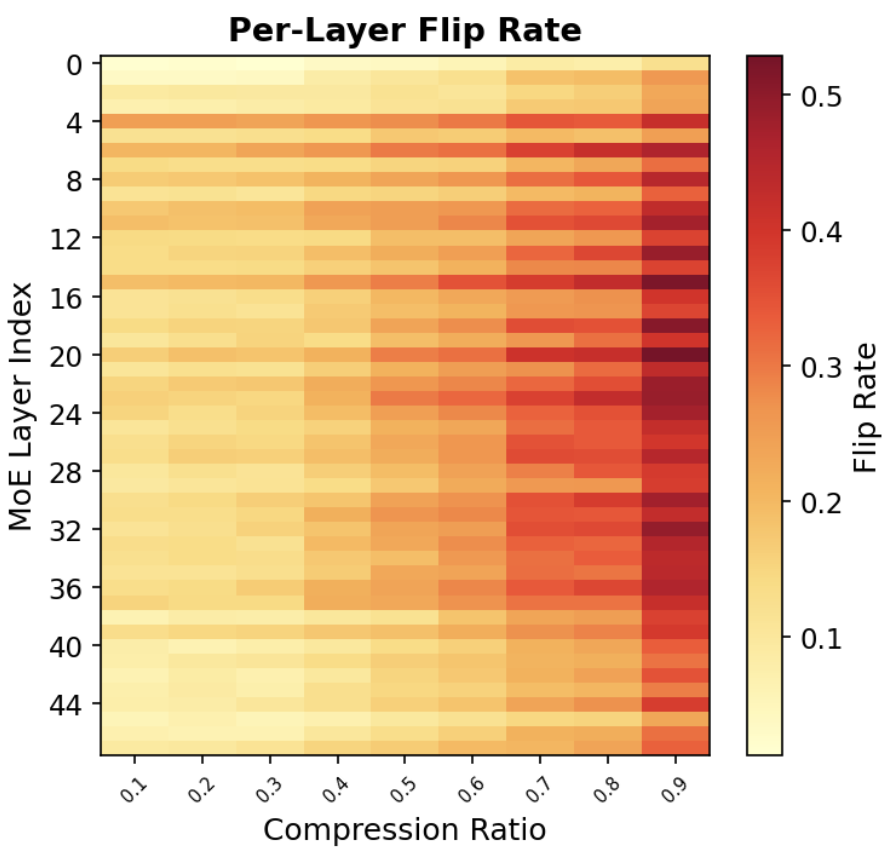}
        \caption{\textbf{Layer-wise router flip rate of SnapKV across compression ratios, averaged over Qasper and HotpotQA.} Similar to KVZip, router instability grows under aggressive compression and is concentrated in the middle-to-late MoE layers.}
        \label{fig:layer_flip_snapkv}
    \end{minipage}
\end{figure}

\begin{figure}[t!]
    \centering
    \begin{minipage}[t]{0.48\linewidth}
        \centering
        \includegraphics[width=\linewidth]{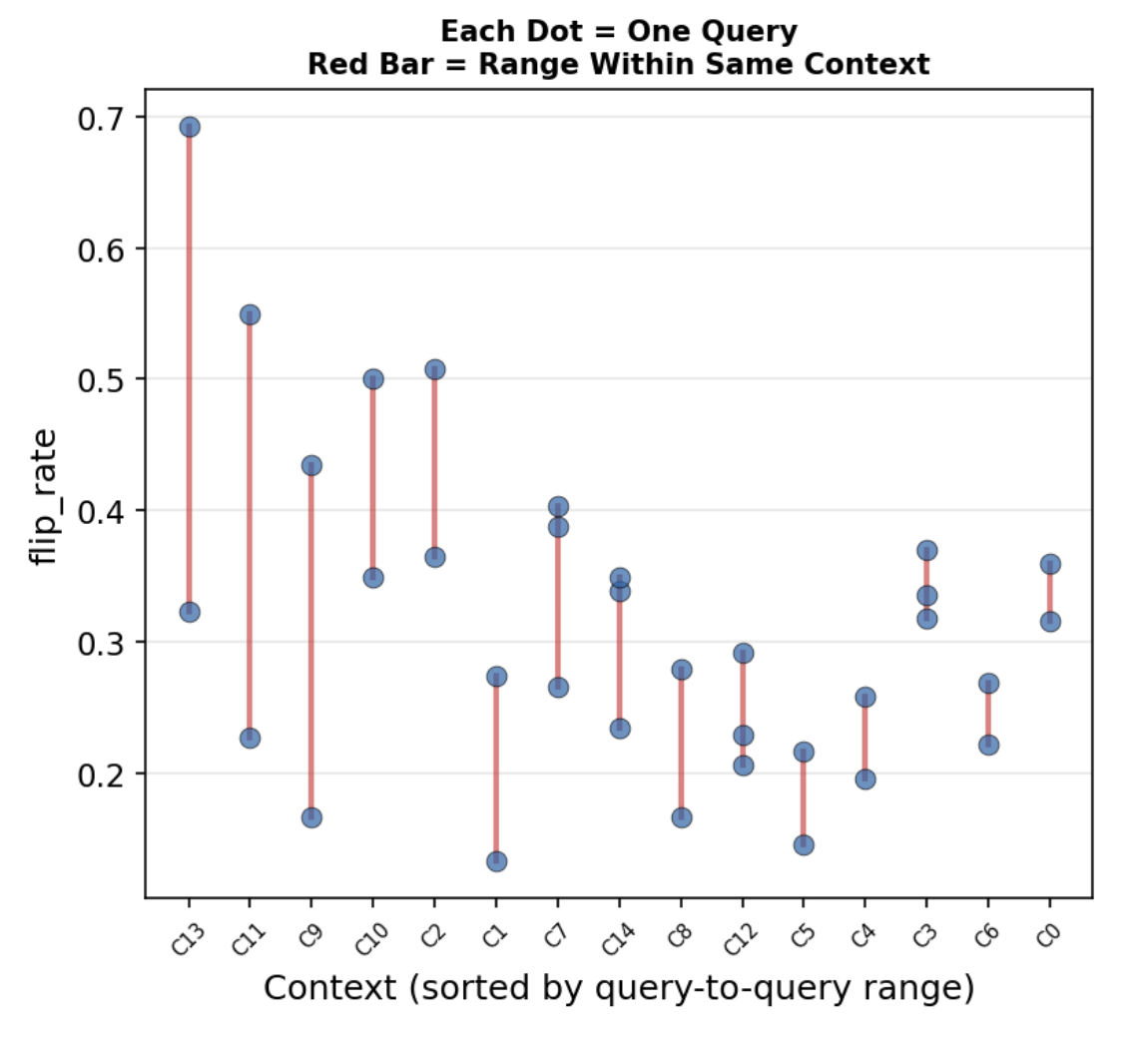}
        \caption{\textbf{Input-wise router flip rates of KVZip on sampled SQuAD v2 contexts.} Each dot denotes one input associated with the same context, and the red bar indicates the within-context range. Even under an identical context, router flip rates vary substantially across inputs, showing that routing instability is strongly input-dependent.}
        \label{fig:query_flip_kvzip}
    \end{minipage}\hfill
    \begin{minipage}[t]{0.48\linewidth}
        \centering
        \includegraphics[width=\linewidth]{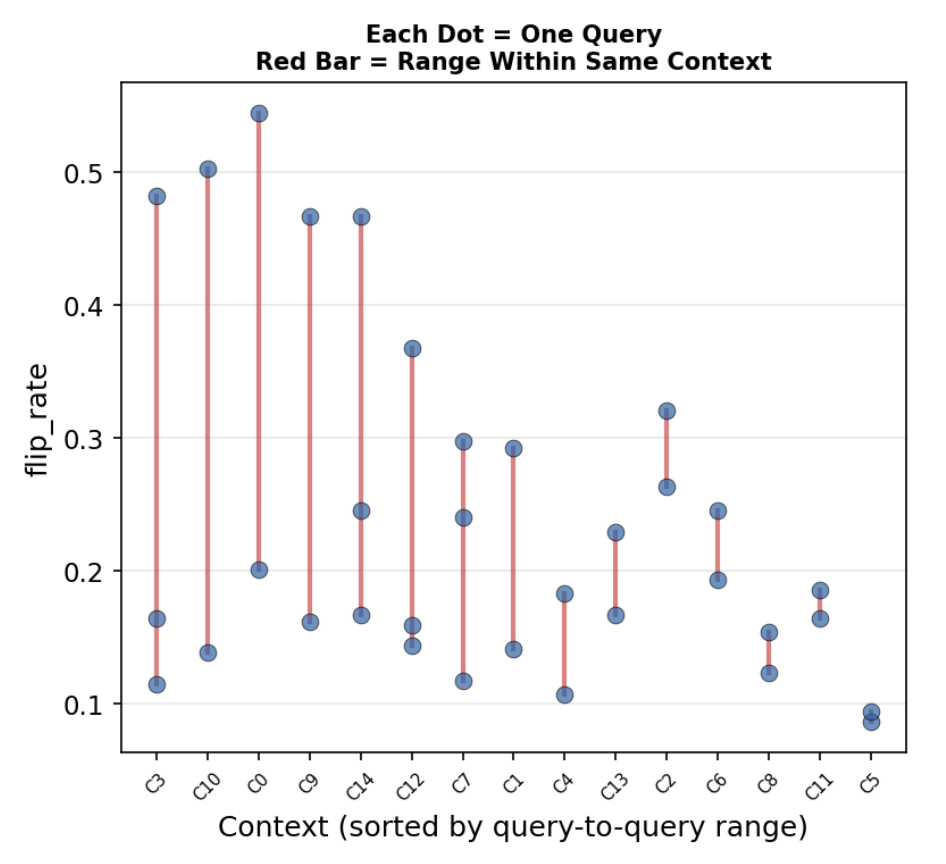}
        \caption{\textbf{Input-wise router flip rates of SnapKV on sampled SQuAD v2 contexts.} Each dot denotes one input associated with the same context, and the red bar indicates the within-context range. Router instability still varies noticeably across inputs sharing the same context, indicating that it depends not only on the context but also on the input.}
        \label{fig:query_flip_snapkv}
    \end{minipage}
\end{figure}

\subsection{Answer-Token Flip Rate per Compression Ratio}
\label{appendix:flip_analysis}

Fig.~\ref{fig:quality_cliff} compares task performance and answer-token expert flip rate across compression ratios on six representative tasks. Consistent with the main-text observation, the flip rate rises as compression becomes more aggressive, and this increase is closely associated with downstream quality degradation. However, the relationship is not merely descriptive: compared with KVZip, MoEZip consistently exhibits lower flip rates across all tasks, and this reduction is accompanied by substantially improved robustness in task performance.

This pattern is especially clear in the high-compression regime. On several tasks, KVZip maintains competitive performance up to a moderate compression level and then exhibits a sharp quality collapse once the flip rate rises sufficiently. By contrast, MoEZip delays this transition by suppressing router instability, thereby preserving answer quality over a broader compression range. These results support our central claim that routing instability is a key failure mode of KV compression in MoE models, and that reducing expert flips is closely tied to improved downstream robustness.

\subsection{Layer-Wise Router Instability under KV Compression}
\label{appendix:layerwise_flip}

We further examine where routing instability emerges across MoE layers. 
For each MoE layer, we compare the top-1 expert selected with the full KV cache and with the compressed KV cache, and report the fraction of tokens whose expert assignment changes.

As shown in Fig.~\ref{fig:layer_flip_kvzip} and Fig.~\ref{fig:layer_flip_snapkv}, both KVZip and SnapKV exhibit a clear increase in router flip rate as compression becomes more aggressive. 
In addition, the instability is not uniform across depth: the flip rate is relatively low in early layers but becomes substantially higher in deeper layers, especially in the middle-to-late MoE layers. 
This suggests that routing perturbations induced by KV eviction are not limited to isolated shallow routing decisions, but can affect expert selection across a broad range of MoE layers.

These results motivate incorporating routing-aware protection into KV compression. 
However, layer-wise flip variation alone does not imply that the routing-aware score should replace fine-grained layer/head-specific KV scoring. 
Instead, MoEZip uses the routing-aware score as a conservative position-level anchor signal: it selects a small set of context positions that are likely to be important for routing stability and protects them across layers. 
The remaining KV budget is still allocated by the layer/head-specific base scorer, allowing MoEZip to combine coarse routing-aware protection with fine-grained KV allocation.

\subsection{Router Flip Rate across Different Inputs for the Same Context}
\label{appendix:query_flip}

We further analyze router instability across different inputs that share the same context. 
For each input, we compare the top-1 expert selected with the full KV cache and with the compressed KV cache, and report the input-wise router flip rate under the same compressed context.

As shown in Fig.~\ref{fig:query_flip_kvzip} and Fig.~\ref{fig:query_flip_snapkv}, both KVZip and SnapKV exhibit substantial within-context variation in router flip rates: some inputs remain relatively stable, while others induce much larger routing disruption under the same context and compression ratio. 
This indicates that routing instability is input-dependent and cannot be fully explained by context-side saliency alone.

This supports the need for instruction-conditioned refinement. 
A context-only score assigns the same importance to context positions regardless of the downstream input, but the routing-critical positions to preserve may vary across inputs. 
MoEZip therefore estimates routing fragility from the instruction tokens and uses instruction-to-context attention to identify input-specific routing-aware anchors.

\subsection{Paired Bootstrap Confidence Intervals}
\label{appendix:CI}

\begin{table}[h!]
\centering
\small
\caption{\textbf{Paired task-level bootstrap intervals for LongBench.} Each entry reports the mean score difference between MoEZip and a baseline, with 95\% percentile intervals over 10{,}000 bootstrap samples. All intervals exclude zero, indicating consistent improvements under task-level resampling.}
\begin{tabular}{lccc}
\toprule
Comparison & Keep 0.1 & Keep 0.3 & Keep 0.5 \\
\midrule
MoEZip $-$ KVZip
& $+25.4$ [$+13.5$, $+36.9$]
& $+8.7$ [$+5.3$, $+11.9$]
& $+0.58$ [$+0.17$, $+1.12$] \\
MoEZip $-$ SnapKV
& $+13.8$ [$+7.8$, $+21.1$]
& $+6.1$ [$+3.3$, $+9.4$]
& $+3.3$ [$+1.8$, $+4.9$] \\
MoEZip $-$ PyramidKV
& $+16.3$ [$+9.2$, $+25.4$]
& $+14.5$ [$+6.7$, $+25.2$]
& $+13.9$ [$+7.5$, $+21.2$] \\
\bottomrule
\end{tabular}
\label{tab:paired_bootstrap_longbench}
\end{table}

\begin{table}[h!]
\centering
\small
\caption{\textbf{Paired task-level bootstrap intervals for RULER.} Each entry reports the mean score difference between MoEZip and a baseline, with 95\% percentile intervals over 10{,}000 bootstrap samples. All intervals exclude zero, indicating consistent improvements under task-level resampling.}
\begin{tabular}{lccc}
\toprule
Comparison & Keep 0.1 & Keep 0.3 & Keep 0.5 \\
\midrule
MoEZip $-$ KVZip
& $+20.4$ [$+9.2$, $+34.9$]
& $+16.1$ [$+6.3$, $+27.9$]
& $+5.6$ [$+0.2$, $+13.3$] \\
MoEZip $-$ SnapKV
& $+9.7$ [$+1.6$, $+25.1$]
& $+55.1$ [$+35.3$, $+73.9$]
& $+53.5$ [$+33.7$, $+72.1$] \\
MoEZip $-$ PyramidKV
& $+9.8$ [$+1.6$, $+25.1$]
& $+60.4$ [$+42.5$, $+76.8$]
& $+71.5$ [$+54.3$, $+87.6$] \\
\bottomrule
\end{tabular}
\label{tab:paired_bootstrap_ruler}
\end{table}

We perform paired task-level bootstrap analysis to assess whether MoEZip's improvements over each baseline are robust to task composition. For each bootstrap sample, we resample tasks with replacement and compute the mean score difference between MoEZip and each baseline on the same resampled task set.
We report 95\% percentile intervals over 10{,}000 bootstrap samples.

As shown in Tables~\ref{tab:paired_bootstrap_longbench} and~\ref{tab:paired_bootstrap_ruler}, MoEZip outperforms all baselines across all reported keep ratios on both LongBench and RULER, with all paired intervals excluding zero. These results indicate that MoEZip's improvements are robust to task-level resampling across the reported aggregate settings.

\subsection{Query-Agnostic Compression and Cache Reuse}
\label{app:query-agnostic}
\label{appendix:query_agnostic}

We compare KVZip, MoEZip-Recon, and instruction-conditioned MoEZip on Qwen3-30B-A3B-Instruct-2507 across 12 LongBench tasks. KVZip and MoEZip-Recon use the same fixed reconstruction query, base scoring procedure, and retained-KV budget, and both produce caches reusable across downstream requests. The original MoEZip instead conditions anchors on the downstream instruction.

\begin{table}[t]
\centering
\small
\setlength{\tabcolsep}{6pt}
\begin{tabular}{lccc}
\toprule
\textbf{Method} & \shortstack{\textbf{Reusable}\\\textbf{cache}}
& \textbf{Keep 0.2} & \textbf{Keep 0.1} \\
\midrule
KVZip        & Yes & 29.25 & 19.51 \\
MoEZip-Recon & Yes & 45.14 & 41.95 \\
MoEZip       & No  & \textbf{49.59} & \textbf{44.89} \\
\bottomrule
\end{tabular}
\caption{\textbf{Query conditioning and compressed-cache reuse.}
LongBench 12-task macro scores on Qwen3-30B-A3B-Instruct-2507.
All methods use matched retained-KV budgets.}
\label{tab:query-agnostic-appendix}
\end{table}

At keep ratio $0.1$, MoEZip-Recon improves over KVZip by $22.44$ points
and remains $2.94$ points below instruction-conditioned MoEZip.
It retains approximately $88\%$ of MoEZip's improvement over KVZip;
the corresponding fraction at keep ratio $0.2$ is $78\%$.
Thus, substantial routing-aware gains persist without downstream
query conditioning, while instruction conditioning provides an
additional quality benefit.
The reuse comparison concerns the final compressed context cache;
it does not claim identical end-to-end costs for the two deployment modes.

\subsection{Global versus Fine-Grained Anchor Selection}
\label{app:anchor-granularity}
\label{appendix:anchor_granularity}

We compare global anchors with per-layer and per-GQA selection on
Qwen3-30B-A3B-Instruct-2507 at keep ratio $0.1$. For the fine-grained variants, local budgets are normalized so that the total number of retained KV elements matches the global method. All variants use the same retention budget $B$ and anchor budget $B_f$ for each layer--head cache, differing only in how anchor positions are scored and selected.

\begin{table}[t]
\centering
\small
\begin{tabular}{lrrr}
\toprule
\textbf{Evaluation} & \textbf{Global}
& \textbf{Per-layer} & \textbf{Per-GQA} \\
\midrule
LongBench 12-task macro & \textbf{44.9} & 36.3 & 39.3 \\
Gap relative to global & \textemdash & $-8.6$ & $-5.6$ \\
Tasks won by global    & \textemdash & 12/12 & 11/12 \\
\bottomrule
\end{tabular}
\caption{\textbf{Anchor-selection granularity under matched KV budgets.}
The base compressor retains layer/head-specific selection
for all non-anchor entries.}
\label{tab:anchor-granularity}
\end{table}

Global selection exceeds per-layer and per-GQA selection by $8.6$
and $5.6$ macro points, respectively.
It wins on all 12 tasks against the per-layer variant and on 11 of 12
tasks against the per-GQA variant.
These results support global sharing among the evaluated alternatives,
but do not establish optimality among all joint cross-layer strategies.
MoEZip shares only routing-aware anchor positions; the remaining budget
is filled independently by the base compressor at each layer and head.

\subsection{Complementarity of Routing-Aware Anchors}
\label{app:anchor-overlap}
\label{appendix:anchor_overlap}

\paragraph{Layer--head-aware overlap.}
Let $\mathcal{A}_{\mathrm{anc}}$ be the global anchor set, with
$|\mathcal{A}_{\mathrm{anc}}|=B_f$, and let
$\mathcal{S}_{\ell,h}^{\mathrm{KVZip}}$ denote the positions retained
by KVZip at layer $\ell$ and KV head $h$.
For the set $\mathcal{C}$ of compressed layer--head caches, define
\begin{equation}
\operatorname{Overlap}
=\frac{1}{|\mathcal{C}|}
\sum_{(\ell,h)\in\mathcal{C}}
\frac{|\mathcal{A}_{\mathrm{anc}}\cap
        \mathcal{S}_{\ell,h}^{\mathrm{KVZip}}|}{B_f}.
\label{eq:anchor-overlap}
\end{equation}
The corresponding mean number of newly protected anchors per cache is
\begin{equation}
\begin{aligned}
\operatorname{New}
&=\frac{1}{|\mathcal{C}|}\sum_{(\ell,h)\in\mathcal{C}}
|\mathcal{A}_{\mathrm{anc}}\setminus
  \mathcal{S}_{\ell,h}^{\mathrm{KVZip}}|\\
&=B_f(1-\operatorname{Overlap}).
\end{aligned}
\label{eq:anchor-new}
\end{equation}
This accounts for a position retained in one cache but absent from others.

\paragraph{Results.}
We evaluate Qwen3-30B-A3B-Instruct-2507 at keep ratio $0.1$ with $B_f=128$.
KVZip already retains only $15$--$24\%$ of the anchors in an average
layer--head cache. Thus, approximately $76$--$85\%$ of anchor slots
protect positions not already retained by KVZip in that cache.

\begin{table}[t]
\centering
\small
\begin{tabular}{lrr}
\toprule
\textbf{Task} & \textbf{Overlap} & \textbf{New / 128} \\
\midrule
RULER QA-1       & 15\% & 108 \\
NIAH-single-2    & 16\% & 108 \\
NIAH-multikey-1  & 16\% & 107 \\
NIAH-multiquery  & 16\% & 107 \\
LCC             & 17\% & 106 \\
TREC            & 20\% & 103 \\
Qasper          & 19\% & 104 \\
PsgRetrieval    & 19\% & 104 \\
HotpotQA        & 24\% & 97 \\
\bottomrule
\end{tabular}
\caption{\textbf{Layer--head-aware anchor overlap at $B_f=128$.}
New / 128 is the mean number of anchor positions absent from
the corresponding KVZip cache.
Percentages and counts are rounded independently.}
\label{tab:anchor-overlap}
\end{table}

\paragraph{Complementarity and utility.}
Anchors occupy part of the existing KV budget rather than adding entries.
Overlapping positions are counted once, and the base scorer fills
the remaining slots.
Low overlap establishes different protection patterns, but does not
by itself establish utility.
The matched query-agnostic comparison in
Appendix~\ref{app:query-agnostic} provides complementary evidence:
MoEZip-Recon improves the macro from $19.51$ to $41.95$ under the
same reconstruction query and reusable-cache setting.

\subsection{Evaluation on a 235B MoE Model}
\label{app:large-model}
\label{appendix:large_model}

We evaluate Qwen3-235B-A22B-Instruct-2507 at keep ratio $0.1$.
Relative to Qwen3-30B-A3B, this increases total and active parameter
counts while retaining the Qwen3 family and softmax-routing design.
It therefore evaluates transfer across scales within the same family. The 235B-model experiments used eight NVIDIA A100 GPUs with 80GB memory each.

\paragraph{Comparison with attention-only selection.}
On a four-task subset, MoEZip improves the macro from $16.7$ with
KVZip to $35.5$ and exceeds attention-only selection by $11.4$ points. The latter comparison supports a contribution beyond instruction-conditioned attention alone.

\begin{table}[t]
\centering
\small
\setlength{\tabcolsep}{4pt}
\begin{tabular}{lrrrrr}
\toprule
\textbf{Method} & \textbf{HotpotQA} & \textbf{PsgRet}
& \textbf{Qasper} & \textbf{QMSUM} & \textbf{Macro} \\
\midrule
Uncompressed   & 57.0 & 100.0 & 46.9 & 23.6 & 56.9 \\
KVZip          & 35.2 & 4.0   & 10.9 & 16.8 & 16.7 \\
Attention-only & 31.4 & 38.8  & 11.0 & 15.0 & 24.1 \\
MoEZip & \textbf{45.6} & \textbf{61.5}
& \textbf{17.5} & \textbf{17.3} & \textbf{35.5} \\
\bottomrule
\end{tabular}
\caption{\textbf{Four-task evaluation on Qwen3-235B-A22B.}
Compressed methods use keep ratio $0.1$.
Bold indicates the best compressed result; the uncompressed model
is a reference.}
\label{tab:235b-controls}
\end{table}

\paragraph{Broader task coverage.}
Across 10 LongBench tasks, MoEZip improves the macro from $27.6$
to $44.7$ and outperforms KVZip on every task.
Gains extend across retrieval, classification, code completion,
and question answering.

\begin{table}[t]
\centering
\small
\setlength{\tabcolsep}{7pt}
\begin{tabular}{lrrr}
\toprule
\textbf{Task} & \textbf{KVZip} & \textbf{MoEZip}
& $\boldsymbol{\Delta}$ \\
\midrule
PsgRet       & 4.0  & \textbf{61.5} & $+57.5$ \\
TREC         & 16.0 & \textbf{50.0} & $+34.0$ \\
LCC          & 28.8 & \textbf{55.9} & $+27.1$ \\
MultiFieldQA & 20.7 & \textbf{35.6} & $+14.8$ \\
TriviaQA     & 75.5 & \textbf{86.9} & $+11.4$ \\
HotpotQA     & 35.2 & \textbf{45.6} & $+10.5$ \\
2WikiMQA     & 24.7 & \textbf{30.2} & $+5.5$ \\
NarrativeQA  & 22.4 & \textbf{26.3} & $+3.9$ \\
SAMSum       & 31.1 & \textbf{34.2} & $+3.1$ \\
MuSiQue      & 17.4 & \textbf{20.3} & $+2.9$ \\
\midrule
10-task macro & 27.6 & \textbf{44.7} & $+17.1$ \\
\bottomrule
\end{tabular}
\caption{\textbf{Ten-task evaluation on Qwen3-235B-A22B at keep ratio $0.1$.}
Deltas use unrounded task scores and can differ slightly from subtraction
of displayed values. This task set differs from the four-task control
comparison in Table~\ref{tab:235b-controls}.}
\label{tab:235b-ten-task}
\end{table}

\paragraph{Routing-signal transfer.}
We evaluate Fisher-based ranking on HotpotQA and PassageRetrieval-en,
forming quintiles separately within each task and layer. The highest Fisher quintile exhibits $2.8\times$ and $2.5\times$ the flip rate of the lowest quintile on Qwen3-30B and Qwen3-235B, respectively. Instruction--answer anchor agreement is also comparable across scales. Together, these diagnostics support transfer of both routing-related signals to the larger model under aggressive compression. This diagnostic uses 50 sampled examples per task. The sampled subsets differ from those used in Appendix~\ref{app:fisher-aggressive}.

\begin{table}[t]
\centering
\small
\setlength{\tabcolsep}{5pt}
\begin{tabular}{lcc}
\toprule
\textbf{Diagnostic}
& \textbf{Qwen3-30B}
& \textbf{Qwen3-235B} \\
\midrule
Fisher Q5/Q1 flip-rate ratio
& $2.8\times$ & $2.5\times$ \\
Anchor Jaccard: HotpotQA
& .88 & .92 \\
Anchor Jaccard: PsgRetrieval
& .96 & .97 \\
\bottomrule
\end{tabular}
\caption{\textbf{Routing-signal diagnostics across model scales.}
Results at keep ratio $0.1$.
Fisher quintiles are formed within each task and layer,
with flip rates aggregated over HotpotQA and PassageRetrieval-en.
Jaccard similarity compares instruction- and answer-conditioned
anchor sets.}
\label{tab:235b-anchor-agreement}
\end{table}

\subsection{Composition with INT8 KV Quantization}
\label{app:kv-quantization}
\label{appendix:kv_quantization}

\paragraph{Setup.}
We evaluate Qwen3-30B-A3B-Instruct-2507 at keep ratio $0.1$ on a six-task LongBench subset (Qasper, TREC, QMSUM, 2WikiMQA, MuSiQue, and SAMSUM). KVZip and MoEZip use the same HQQ INT8 cache-quantization configuration. Within each method, retained positions are fixed between BF16 and INT8. Routing-aware anchors and base-selected entries are quantized identically. The implementation uses Transformers 5.0.0, \texttt{DynamicCache} for BF16, and the HQQ-backed \texttt{HQQQuantizedCache} for INT8. Keys and values use 8-bit quantization with \texttt{q\_group\_size=64}.

\begin{table}[t]
\centering
\small
\setlength{\tabcolsep}{7pt}
\begin{tabular}{lrrr}
\toprule
\textbf{Method / quantity} & \textbf{BF16}
& \textbf{INT8} & \textbf{Change} \\
\midrule
KVZip  & 12.8 & 12.5 & $-0.3$ \\
MoEZip & \textbf{26.4} & \textbf{25.8} & $-0.6$ \\
\midrule
MoEZip $-$ KVZip & $+13.6$ & $+13.3$ & $-0.3$ \\
\bottomrule
\end{tabular}
\caption{\textbf{Composition with INT8 KV quantization.}
Six-task LongBench macro scores at keep ratio $0.1$.
Within each method, BF16 and INT8 retain the same context positions;
anchors receive the same quantization treatment as other retained entries.}
\label{tab:int8-composition}
\end{table}

INT8 reduces MoEZip's macro by $0.6$ points.
Its advantage over KVZip changes from $13.6$ to $13.3$ points,
preserving approximately $97.8\%$ of the BF16 gain.
This supports compatibility with the evaluated INT8 configuration;
lower-bit quantization is not evaluated here.

\subsection{Cross-Layer Routing Associations and Anchor Examples}
\label{app:cross-layer-association}
\label{appendix:cross_layer_association}

\paragraph{Adjacent-layer associations.} We analyze Qwen3-30B-A3B-Instruct-2507 on 12 LongBench tasks with 30 samples per task and 128 globally shared routing-aware anchors. The same reference answer is teacher-forced under full-KV and
compressed-KV inference.
Let $F_{t,\ell}$ indicate that the selected expert assignment for
answer token $t$ differs between these runs at MoE layer $\ell$.

We use a KV keep ratio of $0.1$. We pool all evaluated answer-token observations across tasks and adjacent MoE-layer pairs, then compute each conditional probability
from the pooled counts.

\begin{table}[t]
\centering
\small
\begin{tabular}{lc}
\toprule
\textbf{Conditional probability} & \textbf{Value} \\
\midrule
$\Pr(F_{t,\ell+1}=1\mid F_{t,\ell}=1)$ & .448 \\
$\Pr(F_{t,\ell+1}=1\mid F_{t,\ell}=0)$ & .270 \\
\bottomrule
\end{tabular}
\caption{\textbf{Association between routing flips at adjacent MoE layers.} Results at KV keep ratio $0.1$. Conditional probabilities are computed from pooled answer-token
observations across all evaluated tasks and adjacent MoE-layer pairs.}
\label{tab:cross-layer-association}
\end{table}

The probability of a flip at the next layer is higher when a flip
occurs at the current layer.
This establishes an association, but does not isolate causal propagation
from shared sensitivity to the same KV perturbation. It also does not by itself explain the advantage of globally shared anchors.

\paragraph{Qualitative anchor examples.}
Selected anchors include answer-relevant entities and factual details, such as \texttt{Pelham Glassford}, \texttt{Meg Ryan},
\texttt{John Chilton}, and numerical or temporal expressions.
Other examples include schema markers such as \texttt{Type:},
\texttt{Question:}, and \texttt{yes/no/unanswerable}, as well as evidence
delimiters such as \texttt{Passage 5:}, \texttt{Paragraph 1:},
and \texttt{Article:}.
These examples illustrate positions receiving global protection;
they are qualitative observations rather than a quantitative
classification or causal explanation of anchor utility.

%% file: checklist.tex
\clearpage
\section*{NeurIPS Paper Checklist}

\begin{enumerate}

\item {\bf Claims}
    \item[] Question: Do the main claims made in the abstract and introduction accurately reflect the paper's contributions and scope?
    \item[] Answer: \answerYes{} 
    \item[] Justification: \textbf{The abstract and introduction state the paper's main claims: identifying routing instability as a MoE-specific failure mode of KV compression and proposing MoEZip to improve the accuracy--efficiency trade-off. The scope is limited to sparse MoE LLM inference under KV cache compression, and the main empirical claims are supported by experiments on LongBench, NIAH, and RULER.}
    \item[] Guidelines:
    \begin{itemize}
        \item The answer \answerNA{} means that the abstract and introduction do not include the claims made in the paper.
        \item The abstract and/or introduction should clearly state the claims made, including the contributions made in the paper and important assumptions and limitations. A \answerNo{} or \answerNA{} answer to this question will not be perceived well by the reviewers. 
        \item The claims made should match theoretical and experimental results, and reflect how much the results can be expected to generalize to other settings. 
        \item It is fine to include aspirational goals as motivation as long as it is clear that these goals are not attained by the paper. 
    \end{itemize}

\item {\bf Limitations}
    \item[] Question: Does the paper discuss the limitations of the work performed by the authors?
    \item[] Answer: \answerYes{} 
    \item[] Justification: \textbf{We discuss limitations in Section~\ref{sec:conclusion}, including the need to access router outputs and token-to-context attention statistics, and the fact that Fisher trace is a practical proxy rather than an exact predictor of eviction-induced perturbations.}
    \item[] Guidelines:
    \begin{itemize}
        \item The answer \answerNA{} means that the paper has no limitation while the answer \answerNo{} means that the paper has limitations, but those are not discussed in the paper. 
        \item The authors are encouraged to create a separate ``Limitations'' section in their paper.
        \item The paper should point out any strong assumptions and how robust the results are to violations of these assumptions (e.g., independence assumptions, noiseless settings, model well-specification, asymptotic approximations only holding locally). The authors should reflect on how these assumptions might be violated in practice and what the implications would be.
        \item The authors should reflect on the scope of the claims made, e.g., if the approach was only tested on a few datasets or with a few runs. In general, empirical results often depend on implicit assumptions, which should be articulated.
        \item The authors should reflect on the factors that influence the performance of the approach. For example, a facial recognition algorithm may perform poorly when image resolution is low or images are taken in low lighting. Or a speech-to-text system might not be used reliably to provide closed captions for online lectures because it fails to handle technical jargon.
        \item The authors should discuss the computational efficiency of the proposed algorithms and how they scale with dataset size.
        \item If applicable, the authors should discuss possible limitations of their approach to address problems of privacy and fairness.
        \item While the authors might fear that complete honesty about limitations might be used by reviewers as grounds for rejection, a worse outcome might be that reviewers discover limitations that aren't acknowledged in the paper. The authors should use their best judgment and recognize that individual actions in favor of transparency play an important role in developing norms that preserve the integrity of the community. Reviewers will be specifically instructed to not penalize honesty concerning limitations.
    \end{itemize}

\item {\bf Theory assumptions and proofs}
    \item[] Question: For each theoretical result, does the paper provide the full set of assumptions and a complete (and correct) proof?
    \item[] Answer: \answerYes{} 
    \item[] Justification: \textbf{Proposition~\ref{prop:fisher_trace} states the assumptions on the perturbation distribution and local KL expansion, and the full proof is provided in Appendix~\ref{appendix:router_fisher}. Additional derivations for worst-case curvature and router-specific Fisher forms are provided in the appendix.}

    \item[] Guidelines:
    \begin{itemize}
        \item The answer \answerNA{} means that the paper does not include theoretical results. 
        \item All the theorems, formulas, and proofs in the paper should be numbered and cross-referenced.
        \item All assumptions should be clearly stated or referenced in the statement of any theorems.
        \item The proofs can either appear in the main paper or the supplemental material, but if they appear in the supplemental material, the authors are encouraged to provide a short proof sketch to provide intuition. 
        \item Inversely, any informal proof provided in the core of the paper should be complemented by formal proofs provided in appendix or supplemental material.
        \item Theorems and Lemmas that the proof relies upon should be properly referenced. 
    \end{itemize}

    \item {\bf Experimental result reproducibility}
    \item[] Question: Does the paper fully disclose all the information needed to reproduce the main experimental results of the paper to the extent that it affects the main claims and/or conclusions of the paper (regardless of whether the code and data are provided or not)?
    \item[] Answer: \answerYes{} 
    \item[] Justification: \textbf{The paper describes the evaluated models, datasets, baselines, KV keep ratios, compression protocol, decoding setup, and implementation details in Section~\ref{sec:experiment} and Appendix~\ref{appendix:implementation}. The method equations specify how routing-aware scores, adaptive anchor budgets, and final KV selection are computed.}
    \item[] Guidelines:
    \begin{itemize}
        \item The answer \answerNA{} means that the paper does not include experiments.
        \item If the paper includes experiments, a \answerNo{} answer to this question will not be perceived well by the reviewers: Making the paper reproducible is important, regardless of whether the code and data are provided or not.
        \item If the contribution is a dataset and\slash or model, the authors should describe the steps taken to make their results reproducible or verifiable. 
        \item Depending on the contribution, reproducibility can be accomplished in various ways. For example, if the contribution is a novel architecture, describing the architecture fully might suffice, or if the contribution is a specific model and empirical evaluation, it may be necessary to either make it possible for others to replicate the model with the same dataset, or provide access to the model. In general. releasing code and data is often one good way to accomplish this, but reproducibility can also be provided via detailed instructions for how to replicate the results, access to a hosted model (e.g., in the case of a large language model), releasing of a model checkpoint, or other means that are appropriate to the research performed.
        \item While NeurIPS does not require releasing code, the conference does require all submissions to provide some reasonable avenue for reproducibility, which may depend on the nature of the contribution. For example
        \begin{enumerate}
            \item If the contribution is primarily a new algorithm, the paper should make it clear how to reproduce that algorithm.
            \item If the contribution is primarily a new model architecture, the paper should describe the architecture clearly and fully.
            \item If the contribution is a new model (e.g., a large language model), then there should either be a way to access this model for reproducing the results or a way to reproduce the model (e.g., with an open-source dataset or instructions for how to construct the dataset).
            \item We recognize that reproducibility may be tricky in some cases, in which case authors are welcome to describe the particular way they provide for reproducibility. In the case of closed-source models, it may be that access to the model is limited in some way (e.g., to registered users), but it should be possible for other researchers to have some path to reproducing or verifying the results.
        \end{enumerate}
    \end{itemize}

\item {\bf Open access to data and code}
    \item[] Question: Does the paper provide open access to the data and code, with sufficient instructions to faithfully reproduce the main experimental results, as described in supplemental material?
    \item[] Answer: \answerYes{} 
    \item[] Justification: \textbf{We will provide anonymized code and scripts for reproducing the main experiments and baselines in the supplemental material. The experiments use publicly available benchmarks and open model checkpoints, and data access/preparation details are described in Appendix~\ref{appendix:implementation}.}
    \item[] Guidelines:
    \begin{itemize}
        \item The answer \answerNA{} means that paper does not include experiments requiring code.
        \item Please see the NeurIPS code and data submission guidelines (\url{https://neurips.cc/public/guides/CodeSubmissionPolicy}) for more details.
        \item While we encourage the release of code and data, we understand that this might not be possible, so \answerNo{} is an acceptable answer. Papers cannot be rejected simply for not including code, unless this is central to the contribution (e.g., for a new open-source benchmark).
        \item The instructions should contain the exact command and environment needed to run to reproduce the results. See the NeurIPS code and data submission guidelines (\url{https://neurips.cc/public/guides/CodeSubmissionPolicy}) for more details.
        \item The authors should provide instructions on data access and preparation, including how to access the raw data, preprocessed data, intermediate data, and generated data, etc.
        \item The authors should provide scripts to reproduce all experimental results for the new proposed method and baselines. If only a subset of experiments are reproducible, they should state which ones are omitted from the script and why.
        \item At submission time, to preserve anonymity, the authors should release anonymized versions (if applicable).
        \item Providing as much information as possible in supplemental material (appended to the paper) is recommended, but including URLs to data and code is permitted.
    \end{itemize}

\item {\bf Experimental setting/details}
    \item[] Question: Does the paper specify all the training and test details (e.g., data splits, hyperparameters, how they were chosen, type of optimizer) necessary to understand the results?
    \item[] Answer: \answerYes{} 
    \item[] Justification: \textbf{The paper specifies the evaluation benchmarks, model backbones, baselines, KV keep ratios, compression setup, and evaluation metrics in Section~\ref{sec:experiment}. Since this is an inference-time KV compression method, no model training or optimizer is involved; implementation and hyperparameter details are provided in Appendix~\ref{appendix:implementation}.}
    \item[] Guidelines:
    \begin{itemize}
        \item The answer \answerNA{} means that the paper does not include experiments.
        \item The experimental setting should be presented in the core of the paper to a level of detail that is necessary to appreciate the results and make sense of them.
        \item The full details can be provided either with the code, in appendix, or as supplemental material.
    \end{itemize}

\item {\bf Experiment statistical significance}
    \item[] Question: Does the paper report error bars suitably and correctly defined or other appropriate information about the statistical significance of the experiments?
    \item[] Answer: \answerYes{} 
    \item[] Justification: \textbf{We report paired task-level bootstrap confidence intervals for representative LongBench and RULER aggregate results in Appendix~\ref{appendix:CI}. We compute these intervals by resampling tasks with replacement and recomputing the paired mean score differences between MoEZip and each baseline over 10{,}000 bootstrap samples. We do not perform repeated full inference runs due to the computational cost of evaluating multiple large MoE LLMs across many tasks, baselines, and KV-cache ratios.}
    \item[] Guidelines:
    \begin{itemize}
        \item The answer \answerNA{} means that the paper does not include experiments.
        \item The authors should answer \answerYes{} if the results are accompanied by error bars, confidence intervals, or statistical significance tests, at least for the experiments that support the main claims of the paper.
        \item The factors of variability that the error bars are capturing should be clearly stated (for example, train/test split, initialization, random drawing of some parameter, or overall run with given experimental conditions).
        \item The method for calculating the error bars should be explained (closed form formula, call to a library function, bootstrap, etc.)
        \item The assumptions made should be given (e.g., Normally distributed errors).
        \item It should be clear whether the error bar is the standard deviation or the standard error of the mean.
        \item It is OK to report 1-sigma error bars, but one should state it. The authors should preferably report a 2-sigma error bar than state that they have a 96\% CI, if the hypothesis of Normality of errors is not verified.
        \item For asymmetric distributions, the authors should be careful not to show in tables or figures symmetric error bars that would yield results that are out of range (e.g., negative error rates).
        \item If error bars are reported in tables or plots, the authors should explain in the text how they were calculated and reference the corresponding figures or tables in the text.
    \end{itemize}

\item {\bf Experiments compute resources}
    \item[] Question: For each experiment, does the paper provide sufficient information on the computer resources (type of compute workers, memory, time of execution) needed to reproduce the experiments?
    \item[] Answer: \answerYes{} 
    \item[] Justification: \textbf{We report the compute environment in Appendix~\ref{appendix:resource}. All experiments were conducted on NVIDIA A100 80GB GPUs; we provide the number of GPUs used per run and an approximate GPU-hour estimate for the reported evaluations.}
    \item[] Guidelines:
    \begin{itemize}
        \item The answer \answerNA{} means that the paper does not include experiments.
        \item The paper should indicate the type of compute workers CPU or GPU, internal cluster, or cloud provider, including relevant memory and storage.
        \item The paper should provide the amount of compute required for each of the individual experimental runs as well as estimate the total compute. 
        \item The paper should disclose whether the full research project required more compute than the experiments reported in the paper (e.g., preliminary or failed experiments that didn't make it into the paper). 
    \end{itemize}
    
\item {\bf Code of ethics}
    \item[] Question: Does the research conducted in the paper conform, in every respect, with the NeurIPS Code of Ethics \url{https://neurips.cc/public/EthicsGuidelines}?
    \item[] Answer: \answerYes{} 
    \item[] Justification: \textbf{The research uses publicly available models and benchmarks for inference-time evaluation and does not involve human subjects, private data, or data collection from users. We have reviewed the NeurIPS Code of Ethics and believe the work conforms to it.}
    \item[] Guidelines:
    \begin{itemize}
        \item The answer \answerNA{} means that the authors have not reviewed the NeurIPS Code of Ethics.
        \item If the authors answer \answerNo, they should explain the special circumstances that require a deviation from the Code of Ethics.
        \item The authors should make sure to preserve anonymity (e.g., if there is a special consideration due to laws or regulations in their jurisdiction).
    \end{itemize}

\item {\bf Broader impacts}
    \item[] Question: Does the paper discuss both potential positive societal impacts and negative societal impacts of the work performed?
    \item[] Answer: \answerYes{} 
    \item[] Justification: \textbf{The work may reduce memory and compute requirements for long-context MoE inference, which can improve accessibility and energy efficiency. As with general LLM efficiency improvements, it may also lower the cost of deploying LLMs for both beneficial and potentially harmful applications. We discuss this dual-use consideration in Section~\ref{sec:conclusion}.}
    \item[] Guidelines:
    \begin{itemize}
        \item The answer \answerNA{} means that there is no societal impact of the work performed.
        \item If the authors answer \answerNA{} or \answerNo, they should explain why their work has no societal impact or why the paper does not address societal impact.
        \item Examples of negative societal impacts include potential malicious or unintended uses (e.g., disinformation, generating fake profiles, surveillance), fairness considerations (e.g., deployment of technologies that could make decisions that unfairly impact specific groups), privacy considerations, and security considerations.
        \item The conference expects that many papers will be foundational research and not tied to particular applications, let alone deployments. However, if there is a direct path to any negative applications, the authors should point it out. For example, it is legitimate to point out that an improvement in the quality of generative models could be used to generate Deepfakes for disinformation. On the other hand, it is not needed to point out that a generic algorithm for optimizing neural networks could enable people to train models that generate Deepfakes faster.
        \item The authors should consider possible harms that could arise when the technology is being used as intended and functioning correctly, harms that could arise when the technology is being used as intended but gives incorrect results, and harms following from (intentional or unintentional) misuse of the technology.
        \item If there are negative societal impacts, the authors could also discuss possible mitigation strategies (e.g., gated release of models, providing defenses in addition to attacks, mechanisms for monitoring misuse, mechanisms to monitor how a system learns from feedback over time, improving the efficiency and accessibility of ML).
    \end{itemize}
    
\item {\bf Safeguards}
    \item[] Question: Does the paper describe safeguards that have been put in place for responsible release of data or models that have a high risk for misuse (e.g., pre-trained language models, image generators, or scraped datasets)?
    \item[] Answer: \answerNA{} 
    \item[] Justification: \textbf{The paper does not release new pretrained language models, image generators, scraped datasets, or other high-risk assets. The proposed method is an inference-time KV cache compression algorithm evaluated on existing public models and benchmarks.}
    \item[] Guidelines:
    \begin{itemize}
        \item The answer \answerNA{} means that the paper poses no such risks.
        \item Released models that have a high risk for misuse or dual-use should be released with necessary safeguards to allow for controlled use of the model, for example by requiring that users adhere to usage guidelines or restrictions to access the model or implementing safety filters. 
        \item Datasets that have been scraped from the Internet could pose safety risks. The authors should describe how they avoided releasing unsafe images.
        \item We recognize that providing effective safeguards is challenging, and many papers do not require this, but we encourage authors to take this into account and make a best faith effort.
    \end{itemize}

\item {\bf Licenses for existing assets}
    \item[] Question: Are the creators or original owners of assets (e.g., code, data, models), used in the paper, properly credited and are the license and terms of use explicitly mentioned and properly respected?
    \item[] Answer: \answerYes{} 
    \item[] Justification: \textbf{We cite the papers corresponding to the models, datasets, and baseline methods used in the experiments. The experiments rely on publicly available benchmarks and model checkpoints, and we follow their intended research-use settings.}
    \item[] Guidelines:
    \begin{itemize}
        \item The answer \answerNA{} means that the paper does not use existing assets.
        \item The authors should cite the original paper that produced the code package or dataset.
        \item The authors should state which version of the asset is used and, if possible, include a URL.
        \item The name of the license (e.g., CC-BY 4.0) should be included for each asset.
        \item For scraped data from a particular source (e.g., website), the copyright and terms of service of that source should be provided.
        \item If assets are released, the license, copyright information, and terms of use in the package should be provided. For popular datasets, \url{paperswithcode.com/datasets} has curated licenses for some datasets. Their licensing guide can help determine the license of a dataset.
        \item For existing datasets that are re-packaged, both the original license and the license of the derived asset (if it has changed) should be provided.
        \item If this information is not available online, the authors are encouraged to reach out to the asset's creators.
    \end{itemize}

\item {\bf New assets}
    \item[] Question: Are new assets introduced in the paper well documented and is the documentation provided alongside the assets?
    \item[] Answer: \answerNA{} 
    \item[] Justification: \textbf{The paper does not introduce a new dataset, benchmark, or pretrained model. If code is released, it will be documented with instructions for reproducing the reported inference-time experiments.}

    \item[] Guidelines:
    \begin{itemize}
        \item The answer \answerNA{} means that the paper does not release new assets.
        \item Researchers should communicate the details of the dataset\slash code\slash model as part of their submissions via structured templates. This includes details about training, license, limitations, etc. 
        \item The paper should discuss whether and how consent was obtained from people whose asset is used.
        \item At submission time, remember to anonymize your assets (if applicable). You can either create an anonymized URL or include an anonymized zip file.
    \end{itemize}

\item {\bf Crowdsourcing and research with human subjects}
    \item[] Question: For crowdsourcing experiments and research with human subjects, does the paper include the full text of instructions given to participants and screenshots, if applicable, as well as details about compensation (if any)? 
    \item[] Answer: \answerNA{} 
    \item[] Justification: \textbf{The paper does not involve crowdsourcing, user studies, or research with human subjects.}
    \item[] Guidelines:
    \begin{itemize}
        \item The answer \answerNA{} means that the paper does not involve crowdsourcing nor research with human subjects.
        \item Including this information in the supplemental material is fine, but if the main contribution of the paper involves human subjects, then as much detail as possible should be included in the main paper. 
        \item According to the NeurIPS Code of Ethics, workers involved in data collection, curation, or other labor should be paid at least the minimum wage in the country of the data collector. 
    \end{itemize}

\item {\bf Institutional review board (IRB) approvals or equivalent for research with human subjects}
    \item[] Question: Does the paper describe potential risks incurred by study participants, whether such risks were disclosed to the subjects, and whether Institutional Review Board (IRB) approvals (or an equivalent approval/review based on the requirements of your country or institution) were obtained?
    \item[] Answer: \answerNA{} 
    \item[] Justification: \textbf{The paper does not involve crowdsourcing, user studies, or research with human subjects, so IRB approval is not applicable.}
    \item[] Guidelines:
    \begin{itemize}
        \item The answer \answerNA{} means that the paper does not involve crowdsourcing nor research with human subjects.
        \item Depending on the country in which research is conducted, IRB approval (or equivalent) may be required for any human subjects research. If you obtained IRB approval, you should clearly state this in the paper. 
        \item We recognize that the procedures for this may vary significantly between institutions and locations, and we expect authors to adhere to the NeurIPS Code of Ethics and the guidelines for their institution. 
        \item For initial submissions, do not include any information that would break anonymity (if applicable), such as the institution conducting the review.
    \end{itemize}

\item {\bf Declaration of LLM usage}
    \item[] Question: Does the paper describe the usage of LLMs if it is an important, original, or non-standard component of the core methods in this research? Note that if the LLM is used only for writing, editing, or formatting purposes and does \emph{not} impact the core methodology, scientific rigor, or originality of the research, declaration is not required.
    \item[] Answer: \answerNA{} 
    \item[] Justification: \textbf{The core method does not use LLMs as an important, original, or non-standard component beyond evaluating existing LLMs as target models for inference-time KV cache compression. Any use of LLM tools for writing, editing, or formatting, if applicable, did not affect the scientific methodology, results, or originality of the work.}
    \item[] Guidelines:
    \begin{itemize}
        \item The answer \answerNA{} means that the core method development in this research does not involve LLMs as any important, original, or non-standard components.
        \item Please refer to our LLM policy in the NeurIPS handbook for what should or should not be described.
    \end{itemize}

\end{enumerate}

%% file: paper.bib
@book{amari2000methods,
  title={Methods of information geometry},
  author={Amari, Shun-ichi and Nagaoka, Hiroshi},
  volume={191},
  year={2000},
  publisher={American Mathematical Soc.}
}

@article{li2024snapkv,
  title={Snapkv: Llm knows what you are looking for before generation},
  author={Li, Yuhong and Huang, Yingbing and Yang, Bowen and Venkitesh, Bharat and Locatelli, Acyr and Ye, Hanchen and Cai, Tianle and Lewis, Patrick and Chen, Deming},
  journal={Advances in Neural Information Processing Systems},
  volume={37},
  pages={22947--22970},
  year={2024}
}

@article{feng2024ada,
  title={Ada-kv: Optimizing kv cache eviction by adaptive budget allocation for efficient llm inference},
  author={Feng, Yuan and Lv, Junlin and Cao, Yukun and Xie, Xike and Zhou, S Kevin},
  journal={arXiv preprint arXiv:2407.11550},
  year={2024}
}

@article{kim2025kvzip,
  title={Kvzip: Query-agnostic kv cache compression with context reconstruction},
  author={Kim, Jang-Hyun and Kim, Jinuk and Kwon, Sangwoo and Lee, Jae W and Yun, Sangdoo and Song, Hyun Oh},
  journal={arXiv preprint arXiv:2505.23416},
  year={2025}
}

@article{cai2024pyramidkv,
  title={Pyramidkv: Dynamic kv cache compression based on pyramidal information funneling},
  author={Cai, Zefan and Zhang, Yichi and Gao, Bofei and Liu, Yuliang and Li, Yucheng and Liu, Tianyu and Lu, Keming and Xiong, Wayne and Dong, Yue and Hu, Junjie and others},
  journal={arXiv preprint arXiv:2406.02069},
  year={2024}
}

@inproceedings{bai2024longbench,
  title={Longbench: A bilingual, multitask benchmark for long context understanding},
  author={Bai, Yushi and Lv, Xin and Zhang, Jiajie and Lyu, Hongchang and Tang, Jiankai and Huang, Zhidian and Du, Zhengxiao and Liu, Xiao and Zeng, Aohan and Hou, Lei and others},
  booktitle={Proceedings of the 62nd annual meeting of the association for computational linguistics (volume 1: Long papers)},
  pages={3119--3137},
  year={2024}
}

@article{zhang2023h2o,
  title={H2o: Heavy-hitter oracle for efficient generative inference of large language models},
  author={Zhang, Zhenyu and Sheng, Ying and Zhou, Tianyi and Chen, Tianlong and Zheng, Lianmin and Cai, Ruisi and Song, Zhao and Tian, Yuandong and R{\'e}, Christopher and Barrett, Clark and others},
  journal={Advances in Neural Information Processing Systems},
  volume={36},
  pages={34661--34710},
  year={2023}
}

@article{xiao2023efficient,
  title={Efficient streaming language models with attention sinks},
  author={Xiao, Guangxuan and Tian, Yuandong and Chen, Beidi and Han, Song and Lewis, Mike},
  journal={arXiv preprint arXiv:2309.17453},
  year={2023}
}

@article{yang2024gated,
  title={Gated delta networks: Improving mamba2 with delta rule},
  author={Yang, Songlin and Kautz, Jan and Hatamizadeh, Ali},
  journal={arXiv preprint arXiv:2412.06464},
  year={2024}
}

@article{yang2024parallelizing,
  title={Parallelizing linear transformers with the delta rule over sequence length},
  author={Yang, Songlin and Wang, Bailin and Zhang, Yu and Shen, Yikang and Kim, Yoon},
  journal={Advances in neural information processing systems},
  volume={37},
  pages={115491--115522},
  year={2024}
}

@article{qin2024lightning,
  title={Lightning attention-2: A free lunch for handling unlimited sequence lengths in large language models},
  author={Qin, Zhen and Sun, Weigao and Li, Dong and Shen, Xuyang and Sun, Weixuan and Zhong, Yiran},
  journal={arXiv preprint arXiv:2401.04658},
  year={2024}
}

@inproceedings{yuan-etal-2025-native,
    title = "Native Sparse Attention: Hardware-Aligned and Natively Trainable Sparse Attention",
    author = "Yuan, Jingyang  and
      Gao, Huazuo  and
      Dai, Damai  and
      Luo, Junyu  and
      Zhao, Liang  and
      Zhang, Zhengyan  and
      Xie, Zhenda  and
      Wei, Yuxing  and
      Wang, Lean  and
      Xiao, Zhiping  and
      Wang, Yuqing  and
      Ruan, Chong  and
      Zhang, Ming  and
      Liang, Wenfeng  and
      Zeng, Wangding",
    editor = "Che, Wanxiang  and
      Nabende, Joyce  and
      Shutova, Ekaterina  and
      Pilehvar, Mohammad Taher",
    booktitle = "Proceedings of the 63rd Annual Meeting of the Association for Computational Linguistics (Volume 1: Long Papers)",
    month = jul,
    year = "2025",
    address = "Vienna, Austria",
    publisher = "Association for Computational Linguistics",
    url = "https://aclanthology.org/2025.acl-long.1126/",
    doi = "10.18653/v1/2025.acl-long.1126",
    pages = "23078--23097",
    ISBN = "979-8-89176-251-0",
}

@misc{deepseekai2025deepseekv32,
      title={DeepSeek-V3.2: Pushing the Frontier of Open Large Language Models}, 
      author={DeepSeek-AI},
      year={2025},
      eprint={2512.02556},
      archivePrefix={arXiv},
      primaryClass={cs.CL},
      url={https://arxiv.org/abs/2512.02556}, 
}

@inproceedings{zaheer2020bigbird,
  title={Big Bird: Transformers for Longer Sequences},
  author={Zaheer, Manzil and Guruganesh, Guru and Dubey, Avinava and Ainslie, Joshua and Alberti, Chris and Ontanon, Santiago and Pham, Philip and Ravula, Anirudh and Wang, Qifan and Yang, Li and Ahmed, Amr},
  booktitle={Advances in Neural Information Processing Systems (NeurIPS)},
  year={2020}
}

@inproceedings{
xiao2025duoattention,
title={DuoAttention: Efficient Long-Context {LLM} Inference with Retrieval and Streaming Heads},
author={Guangxuan Xiao and Jiaming Tang and Jingwei Zuo and junxian guo and Shang Yang and Haotian Tang and Yao Fu and Song Han},
booktitle={The Thirteenth International Conference on Learning Representations},
year={2025},
url={https://openreview.net/forum?id=cFu7ze7xUm}
}

@article{yang2025qwen3,
  title={Qwen3 technical report},
  author={Yang, An and Li, Anfeng and Yang, Baosong and Zhang, Beichen and Hui, Binyuan and Zheng, Bo and Yu, Bowen and Gao, Chang and Huang, Chengen and Lv, Chenxu and others},
  journal={arXiv preprint arXiv:2505.09388},
  year={2025}
}

@article{zeng2026glm,
  title={GLM-5: from Vibe Coding to Agentic Engineering},
  author={Zeng, Aohan and Lv, Xin and Hou, Zhenyu and Du, Zhengxiao and Zheng, Qinkai and Chen, Bin and Yin, Da and Ge, Chendi and Xie, Chengxing and Wang, Cunxiang and others},
  journal={arXiv preprint arXiv:2602.15763},
  year={2026}
}

@article{li2025minimax,
  title={Minimax-01: Scaling foundation models with lightning attention},
  author={Li, Aonian and Gong, Bangwei and Yang, Bo and Shan, Boji and Liu, Chang and Zhu, Cheng and Zhang, Chunhao and Guo, Congchao and Chen, Da and Li, Dong and others},
  journal={arXiv preprint arXiv:2501.08313},
  year={2025}
}

@article{zeng2025glm,
  title={Glm-4.5: Agentic, reasoning, and coding (arc) foundation models},
  author={Zeng, Aohan and Lv, Xin and Zheng, Qinkai and Hou, Zhenyu and Chen, Bin and Xie, Chengxing and Wang, Cunxiang and Yin, Da and Zeng, Hao and Zhang, Jiajie and others},
  journal={arXiv preprint arXiv:2508.06471},
  year={2025}
}

@article{fu2024data,
  title={Data engineering for scaling language models to 128k context},
  author={Fu, Yao and Panda, Rameswar and Niu, Xinyao and Yue, Xiang and Hajishirzi, Hannaneh and Kim, Yoon and Peng, Hao},
  journal={arXiv preprint arXiv:2402.10171},
  year={2024}
}

@article{hsieh2024ruler,
  title={RULER: What's the real context size of your long-context language models?},
  author={Hsieh, Cheng-Ping and Sun, Simeng and Kriman, Samuel and Acharya, Shantanu and Rekesh, Dima and Jia, Fei and Zhang, Yang and Ginsburg, Boris},
  journal={arXiv preprint arXiv:2404.06654},
  year={2024}
}

@article{achiam2023gpt,
  title={Gpt-4 technical report},
  author={Achiam, Josh and Adler, Steven and Agarwal, Sandhini and Ahmad, Lama and Akkaya, Ilge and Aleman, Florencia Leoni and Almeida, Diogo and Altenschmidt, Janko and Altman, Sam and Anadkat, Shyamal and others},
  journal={arXiv preprint arXiv:2303.08774},
  year={2023}
}

@article{comanici2025gemini,
  title={Gemini 2.5: Pushing the frontier with advanced reasoning, multimodality, long context, and next generation agentic capabilities},
  author={Comanici, Gheorghe and Bieber, Eric and Schaekermann, Mike and Pasupat, Ice and Sachdeva, Noveen and Dhillon, Inderjit and Blistein, Marcel and Ram, Ori and Zhang, Dan and Rosen, Evan and others},
  journal={arXiv preprint arXiv:2507.06261},
  year={2025}
}

@article{bai2022training,
  title={Training a helpful and harmless assistant with reinforcement learning from human feedback},
  author={Bai, Yuntao and Jones, Andy and Ndousse, Kamal and Askell, Amanda and Chen, Anna and DasSarma, Nova and Drain, Dawn and Fort, Stanislav and Ganguli, Deep and Henighan, Tom and others},
  journal={arXiv preprint arXiv:2204.05862},
  year={2022}
}

@article{brown2020language,
  title={Language models are few-shot learners},
  author={Brown, Tom and Mann, Benjamin and Ryder, Nick and Subbiah, Melanie and Kaplan, Jared D and Dhariwal, Prafulla and Neelakantan, Arvind and Shyam, Pranav and Sastry, Girish and Askell, Amanda and others},
  journal={Advances in neural information processing systems},
  volume={33},
  pages={1877--1901},
  year={2020}
}

@article{grattafiori2024llama,
  title={The llama 3 herd of models},
  author={Grattafiori, Aaron and Dubey, Abhimanyu and Jauhri, Abhinav and Pandey, Abhinav and Kadian, Abhishek and Al-Dahle, Ahmad and Letman, Aiesha and Mathur, Akhil and Schelten, Alan and Vaughan, Alex and others},
  journal={arXiv preprint arXiv:2407.21783},
  year={2024}
}

@article{kim2024openvla,
  title={Openvla: An open-source vision-language-action model},
  author={Kim, Moo Jin and Pertsch, Karl and Karamcheti, Siddharth and Xiao, Ted and Balakrishna, Ashwin and Nair, Suraj and Rafailov, Rafael and Foster, Ethan and Lam, Grace and Sanketi, Pannag and others},
  journal={arXiv preprint arXiv:2406.09246},
  year={2024}
}

@article{intelligence2025pi_,
  title={$pi\_\{0.5\}$: a Vision-Language-Action Model with Open-World Generalization},
  author={Intelligence, Physical and Black, Kevin and Brown, Noah and Darpinian, James and Dhabalia, Karan and Driess, Danny and Esmail, Adnan and Equi, Michael and Finn, Chelsea and Fusai, Niccolo and others},
  journal={arXiv preprint arXiv:2504.16054},
  year={2025}
}

@article{shinn2023reflexion,
  title={Reflexion: Language agents with verbal reinforcement learning},
  author={Shinn, Noah and Cassano, Federico and Gopinath, Ashwin and Narasimhan, Karthik and Yao, Shunyu},
  journal={Advances in neural information processing systems},
  volume={36},
  pages={8634--8652},
  year={2023}
}

@article{madaan2023self,
  title={Self-refine: Iterative refinement with self-feedback},
  author={Madaan, Aman and Tandon, Niket and Gupta, Prakhar and Hallinan, Skyler and Gao, Luyu and Wiegreffe, Sarah and Alon, Uri and Dziri, Nouha and Prabhumoye, Shrimai and Yang, Yiming and others},
  journal={Advances in neural information processing systems},
  volume={36},
  pages={46534--46594},
  year={2023}
}

@inproceedings{yao2022react,
  title={React: Synergizing reasoning and acting in language models},
  author={Yao, Shunyu and Zhao, Jeffrey and Yu, Dian and Du, Nan and Shafran, Izhak and Narasimhan, Karthik R and Cao, Yuan},
  booktitle={The eleventh international conference on learning representations},
  year={2022}
}

@article{fedus2022switch,
  title={Switch transformers: Scaling to trillion parameter models with simple and efficient sparsity},
  author={Fedus, William and Zoph, Barret and Shazeer, Noam},
  journal={Journal of Machine Learning Research},
  volume={23},
  number={120},
  pages={1--39},
  year={2022}
}

@inproceedings{du2022glam,
  title={Glam: Efficient scaling of language models with mixture-of-experts},
  author={Du, Nan and Huang, Yanping and Dai, Andrew M and Tong, Simon and Lepikhin, Dmitry and Xu, Yuanzhong and Krikun, Maxim and Zhou, Yanqi and Yu, Adams Wei and Firat, Orhan and others},
  booktitle={International conference on machine learning},
  pages={5547--5569},
  year={2022},
  organization={PMLR}
}

@article{liu2024deepseek,
  title={Deepseek-v3 technical report},
  author={Liu, Aixin and Feng, Bei and Xue, Bing and Wang, Bingxuan and Wu, Bochao and Lu, Chengda and Zhao, Chenggang and Deng, Chengqi and Zhang, Chenyu and Ruan, Chong and others},
  journal={arXiv preprint arXiv:2412.19437},
  year={2024}
}

@article{jimenez2023swe,
  title={Swe-bench: Can language models resolve real-world github issues?},
  author={Jimenez, Carlos E and Yang, John and Wettig, Alexander and Yao, Shunyu and Pei, Kexin and Press, Ofir and Narasimhan, Karthik},
  journal={arXiv preprint arXiv:2310.06770},
  year={2023}
}

@article{zhou2023webarena,
  title={Webarena: A realistic web environment for building autonomous agents},
  author={Zhou, Shuyan and Xu, Frank F and Zhu, Hao and Zhou, Xuhui and Lo, Robert and Sridhar, Abishek and Cheng, Xianyi and Ou, Tianyue and Bisk, Yonatan and Fried, Daniel and others},
  journal={arXiv preprint arXiv:2307.13854},
  year={2023}
}

@article{lewis2020retrieval,
  title={Retrieval-augmented generation for knowledge-intensive nlp tasks},
  author={Lewis, Patrick and Perez, Ethan and Piktus, Aleksandra and Petroni, Fabio and Karpukhin, Vladimir and Goyal, Naman and K{\"u}ttler, Heinrich and Lewis, Mike and Yih, Wen-tau and Rockt{\"a}schel, Tim and others},
  journal={Advances in neural information processing systems},
  volume={33},
  pages={9459--9474},
  year={2020}
}

@article{zheng2024gpt,
  title={Gpt-4v (ision) is a generalist web agent, if grounded},
  author={Zheng, Boyuan and Gou, Boyu and Kil, Jihyung and Sun, Huan and Su, Yu},
  journal={arXiv preprint arXiv:2401.01614},
  year={2024}
}

@article{liu2023libero,
  title={Libero: Benchmarking knowledge transfer for lifelong robot learning},
  author={Liu, Bo and Zhu, Yifeng and Gao, Chongkai and Feng, Yihao and Liu, Qiang and Zhu, Yuke and Stone, Peter},
  journal={Advances in Neural Information Processing Systems},
  volume={36},
  pages={44776--44791},
  year={2023}
}

@article{schick2023toolformer,
  title={Toolformer: Language models can teach themselves to use tools},
  author={Schick, Timo and Dwivedi-Yu, Jane and Dess{\`\i}, Roberto and Raileanu, Roberta and Lomeli, Maria and Hambro, Eric and Zettlemoyer, Luke and Cancedda, Nicola and Scialom, Thomas},
  journal={Advances in neural information processing systems},
  volume={36},
  pages={68539--68551},
  year={2023}
}

@article{zandieh2025turboquant,
  title={Turboquant: Online vector quantization with near-optimal distortion rate},
  author={Zandieh, Amir and Daliri, Majid and Hadian, Majid and Mirrokni, Vahab},
  journal={arXiv preprint arXiv:2504.19874},
  year={2025}
}

@article{team2025kimi,
  title={Kimi k2: Open agentic intelligence},
  author={Team, Kimi and Bai, Yifan and Bao, Yiping and Charles, Y and Chen, Cheng and Chen, Guanduo and Chen, Haiting and Chen, Huarong and Chen, Jiahao and Chen, Ningxin and others},
  journal={arXiv preprint arXiv:2507.20534},
  year={2025}
}

@book{horn2012matrix,
  title={Matrix analysis},
  author={Horn, Roger A and Johnson, Charles R},
  year={2012},
  publisher={Cambridge university press}
}

@article{hyvarinen2005estimation,
  title={Estimation of non-normalized statistical models by score matching.},
  author={Hyv{\"a}rinen, Aapo and Dayan, Peter},
  journal={Journal of Machine Learning Research},
  volume={6},
  number={4},
  year={2005}
}

@misc{qwen25,
  title         = {Qwen2.5 Technical Report},
  author        = {Qwen and Yang, An and Yang, Baosong and Zhang, Beichen and Hui, Binyuan and Zheng, Bo and Yu, Bowen and Li, Chengyuan and Liu, Dayiheng and Huang, Fei and Wei, Haoran and Lin, Huan and Yang, Jian and Tu, Jianhong and Zhang, Jianwei and Yang, Jianxin and Yang, Jiaxi and Zhou, Jingren and Lin, Junyang and Dang, Kai and Lu, Keming and Bao, Keqin and Yang, Kexin and Yu, Le and Li, Mei and Xue, Mingfeng and Zhang, Pei and Zhu, Qin and Men, Rui and Lin, Runji and Li, Tianhao and Tang, Tianyi and Xia, Tingyu and Ren, Xingzhang and Ren, Xuancheng and Fan, Yang and Su, Yang and Zhang, Yichang and Wan, Yu and Liu, Yuqiong and Cui, Zeyu and Zhang, Zhenru and Qiu, Zihan},
  year          = {2024},
  eprint        = {2412.15115},
  archivePrefix = {arXiv},
  primaryClass  = {cs.CL},
  url           = {https://arxiv.org/abs/2412.15115}
}
